\documentclass[11pt]{article}

\usepackage[top=25mm,bottom=30mm,left=30mm,right=30mm]{geometry}
\usepackage{amssymb, amsmath, amsthm}

\usepackage{bm}
\usepackage{enumerate}
\usepackage{algorithm}
\usepackage{graphicx}
\usepackage{booktabs}
\usepackage{url}

\usepackage{color}

\theoremstyle{plain}
\newtheorem{theo}{Theorem}
\newtheorem{prop}{Proposition}
\newtheorem{lemm}{Lemma}
\newtheorem{coro}{Corollary}

\theoremstyle{remark}
\newtheorem{remark}{Remark}

\def\zero{\bm{0}}
\def\one{\bm{1}}

\def\a{\bm{a}}
\def\b{\bm{b}}
\def\c{\bm{c}}
\def\e{\bm{e}}
\def\f{\bm{f}}
\def\g{\bm{g}}
\def\h{\bm{h}}

\def\n{\bm{n}}
\def\p{\bm{p}}
\def\q{\bm{q}}
\def\r{\bm{r}}

\def\u{\bm{u}}

\def\x{\bm{x}}
\def\y{\bm{y}}
\def\z{\bm{z}}

\def\A{\bm{A}}
\def\B{\bm{B}}
\def\C{\bm{C}}
\def\D{\bm{D}}

\def\F{\bm{F}}
\def\G{\bm{G}}
\def\H{\bm{H}}
\def\I{\bm{I}}

\def\L{\bm{L}}
\def\M{\bm{M}}
\def\N{\bm{N}}
\def\P{\bm{P}}
\def\Q{\bm{Q}}
\def\R{\bm{R}}
\def\S{\bm{S}}

\def\U{\bm{U}}
\def\V{\bm{V}}
\def\W{\bm{W}}
\def\X{\bm{X}}
\def\Y{\bm{Y}}

\def\Pib{\bm{\Pi}}

\def\LC{\mathcal{L}}

\def\equivSym{\Leftrightarrow}

\def\trans{\top}
\def\trace{\mbox{tr}}

\def\mmin{\mbox{\scriptsize min}}
\def\mmax{\mbox{\scriptsize max}}

\def\phiIn{\phi_{\mbox{\scriptsize in}}}
\def\phiOut{\phi_{\mbox{\scriptsize out}}}

\def\Real{\mathbb{R}}
\def\Prob{\ensuremath{\mathsf P}}

\newcommand{\by}[2]{\ensuremath{#1 \times #2}}

\title{Convex Programming Based Spectral Clustering}

\author{Tomohiko~Mizutani
\thanks{Department of Mathematical and Systems Engineering,
Shizuoka University,
3-5-1 Johoku, Naka-ku, Hamamatsu City, 432-8561, Japan.
{\tt mizutani.t@shizuoka.ac.jp}}}

\date{\today}

\begin{document}

\maketitle

\begin{abstract}

 Clustering is a fundamental task in data analysis, and 
 spectral clustering has been recognized as a promising approach to it.
 Given a graph describing the relationship between data, 
 spectral clustering explores the underlying cluster structure in two stages.
 The first stage embeds the nodes of the graph in real space,
 and the second stage groups the embedded nodes into several clusters.
 The use of the $k$-means method in the grouping stage is currently standard practice. 
 We present a spectral clustering algorithm that uses convex programming in the grouping stage
 and study how well it works.
 This algorithm is designed based on the following observation.
 If a graph is well-clustered,
 then the nodes with the largest degree in each cluster
 can be found by computing an enclosing ellipsoid of the nodes embedded in real space,
 and the clusters can be identified by using those nodes.
 We show that, for well-clustered graphs, the algorithm can find clusters of nodes with minimal conductance.
 We also give an experimental assessment of the algorithm's performance.

\end{abstract}

\section{Introduction} \label{Sec: intro}
Given a set of items and similarities between pairs of the items,
clustering is the task of partitioning the item set into groups
such that the items within the same group are similar and the items within different groups are dissimilar.
One natural way of representing the task is to use a graph.
We construct a graph such that 
each item corresponds to a node and, 
if a pair of items has  high similarity, there is an edge between them.
The task is cast as one of partitioning the node set into clusters of nodes
such that the nodes within the same cluster are well connected and those within different clusters are poorly connected.

Spectral clustering is a way of finding such clusters in a graph.
It has two stages.
The first stage embeds the nodes of a graph in  real space,
and the second one partitions the embedded nodes into groups.
The embedding uses the eigenvectors of a matrix associated
with the graph, such as the Laplacian.
The grouping employs a classical clustering method such as  $k$-means.
Spectral clustering is said to date back to
the works of Donath and Hoffman \cite{Don73} and Fiedler \cite{Fie73} in the 1970s,
and it was popularized by the works of 
Shi and Malik \cite{Shi00}, Ne et al.\ \cite{Ng02}, Bach and Jordan \cite{Bac03}, Luxburg \cite{Lux07}
in the machine learning and data mining community in the 2000s.
Its effectiveness has been tested on various problems, and
it is now recognized as a promising approach to clustering.

Recently, Peng et al.\ investigated the performance of a spectral clustering algorithm 
that uses the $k$-means method in the grouping stage 
and provided a theoretical justification as to why it works well in practice.
We shall use the abbreviation KSC to refer to
the $k$-means based spectral clustering algorithm.
The results first appeared in the proceedings \cite{Pen15} of COLT 2015, and then
in a journal paper \cite{Pen17}.
They quantified the quality of clusters in a graph in terms of a measurement, conductance.
Clusters with minimal conductance fit the aim of clustering task.
They showed that the output of KSC is a good approximation to such clusters if the input graph is well-clustered.
Later, their results were improved by Kolev and Mehlhorn \cite{Kol16, Kol18}.

A little further explanation may be needed on the results.
Let $G$ be a graph with a node set $V$.
We call a subset $S$ of $V$ a cluster. We call a family of $k$ clusters $S_1, \ldots, S_k$ a $k$-way partition of $G$
if $S_i \cap S_j = \emptyset$ for different $i$ and $j$ and $S_1 \cup \cdots \cup S_k = V$.
The conductance  $\phi(S_i)$ of a cluster $S_i$
is defined to be the ratio of the cut size between $S_i$ and its complement divided by the volume of $S_i$.
We formulate the clustering task on $G$
as the problem of finding a partition $\{S_1, \ldots S_k\}$ of $G$ that minimizes
the maximum of $\phi(S_1), \ldots, \phi(S_k)$.
The $k$-way conductance $\phi_k(G)$ of a graph $G$ is defined to be the minimum value.
We say that a partition $\{S_1, \ldots, S_k\}$ of $G$ is optimal 
if it satisfies $\phi_k(G) = \max \{\phi(S_1), \ldots, \phi(S_k) \}$.
Finding an optimal partition of $G$ is known to be NP-hard \cite{Mat90}.
Let $\lambda_{k+1}$ denote the $(k+1)$ smallest eigenvalue of the normalized Laplacian of $G$.
A graph $G$ is called well-clustered if there is a large gap between $\phi_k(G)$ and $\lambda_{k+1}$;
in other words, $\Upsilon = \lambda_{k+1} / \phi_k(G)$ is large.
In fact, it was shown by Gharan and Trevisan \cite{Gha14} that, if the gap assumption holds, 
there are $k$ clusters in $G$ such that 
the nodes within the same cluster are well connected and those within different clusters are poorly connected.
Peng et al.\ measured how well the output of KSC approximates an optimal $k$-way partition of $G$
and gave approximation guarantees under the assumption that $G$ satisfies $\Upsilon = \Omega(k^3)$.
Later, Kolev and Mehlhorn improved the approximation guarantees under a weaker assumption.

\subsection{Our Contributions} \label{Subsec: contributions}
We present a spectral clustering algorithm that uses convex programming in the grouping stage
and call the algorithm ELLI since an ellipsoid plays an important role in it.
This is built on an extension of the structure theorem shown in \cite{Miz20}. 
The theorem was originally developed by Peng et al.\ \cite{Pen17} 
and it was later extended in \cite{Miz20}.
It implies that, if a graph is well-clustered, then, the nodes with the largest degree in each cluster
can be found by computing an enclosing ellipsoid for the nodes embedded in real space,
and the clusters can be identified by using those nodes.
ELLI is designed on the basis of this observation.
We examine the performance of ELLI from theoretical and practical perspectives.
The main contributions of this study are summarized as follows.

\begin{itemize}
 \item 
       We provide a theoretical analysis of the clustering performance of ELLI.
       In Section \ref{Subsec: results of performance analysis}, we show in Theorem \ref{Theo: main result} that,
       if $\Upsilon$ exceeds some threshold, 
       ELLI returns an optimal $k$-way partition of a graph.
       In contrast to this, no matter how large $\Upsilon$ is,
       the result of Peng et al.\ does not ensure that KSC does so.
       As we saw above, the threshold of Peng et al.\ for KSC depends on only $k$,
       while our threshold for ELLI depends on a graph as well as $k$.
       Thus, in Corollary \ref{Coro: obtained from main result}, 
       we rewrite our threshold using only $k$ and examine how large it can be.

 \item 
       We reveal that algorithms for computing  nonnegative matrix factorizations (NMFs) 
       under the separability condition are useful in the grouping stage of spectral clustering
       In Section \ref{Sec: connection with separable NMFs}, we explain that, if a graph is well-clustered, 
       the algorithms can exactly find the nodes with the largest degree in each cluster.
       This perspective has not been considered so far and provides insight into 
       the design of effective algorithms in the grouping stage.
       We use an ellipsoidal rounding technique developed for 
       computing separable NMFs in \cite{Miz14} and develop ELLI.

 \item
      We present an experimental assessment of ELLI.
      We experimentally tested the effectiveness of ELLI at clustering real data,
      i.e., image datasets whose
      images had been categorized into classes by human judges.
      We applied ELLI to each dataset and
      evaluated how well the clusters found by it matched the classes of the dataset.
      For the evaluation, we used two measures, accuracy (AC) and normalized mutual information (NMI),
      which are commonly used for this purpose.
      The experiments also evaluated the conductance of the clusters found by ELLI.
      We tested two more clustering algorithms of which one was KSC.
      A standard implementation of spectral clustering
      uses the $k$-means method based on Lloyd's algorithm \cite{Llo82}.
      Our implementation of KSC used the $k$-means++  algorithm \cite{Art07}, i.e., its enhancement.
      Since $k$-means++ is probabilistic, we repeated KSC equipped with it multiple times
      and took the average of the measurements for the evaluation of the outputs.
      The experiments revealed that 
      the AC and NMI of ELLI can reach at least the average AC and NMI of KSC.
      The experiments also showed that 
      the conductance of the clusters found by ELLI is often smaller than the value  given by KSC.

\end{itemize}

The rest of this paper is organized as follows.
Section \ref{Sec: preliminaries} explains the notation, symbols, and terminology of graphs, 
which will be used in the subsequent discussion.
It also reviews basic results from spectral graph theory and the spectral clustering algorithm.
Section \ref{Sec: proposed algorithm} explains the details of ELLI 
and shows the theoretical performance in Theorem \ref{Theo: main result}.
Lastly, it describes related work, including
the studies by Peng et al.\ and Kolev and Mehlhorn on the performance of KSC.
Section \ref{Sec: analysis} provides an analysis of ELLI.
Section \ref{Sec: connection with separable NMFs} reviews NMFs and the separability condition.
It then explains why algorithms for computing NMFs 
under the separability condition can be used in the grouping stage of spectral clustering.
Section \ref{Sec: experiments} describes the experimental study.

\medskip \noindent
{\it Notation for vectors and matrices.}
The symbols $\| \cdot \|_1, \| \cdot \|_2$ and $\| \cdot \|_{\infty}$ denote 
the $\ell_1, \ell_2$ and infinity norms of a vector or a matrix.
The symbol $\|\cdot\|_F$ denotes the Frobenius norm of a matrix.
For real numbers $a_1, \ldots, a_n$,
we use $\mbox{diag}(a_1, \ldots, a_n)$ to denote an \by{n}{n} diagonal matrix
having $a_i$ in the $(i,i)$th entry.
We use $\e_i$ to denote the $i$th unit vector and $\I$ to denote the identity matrix.

\begin{remark}
 The previous version of this paper posted on arXiv in 2018 was reviewed by anonymous reviewers.
 Based on their comments and suggestions,  we made major revisions in the current paper.
 In particular, we contained Corollary \ref{Coro: obtained from main result} 
 that was suggested by one of the reviewers;
 the best and worst results of algorithms we tested in Table \ref{Tab: AC, NMI and MCC in 2nd exp};
 the results in case where a neighbor size $p$ was set as $10$ in Figure \ref{Fig: results of 2nd exp};
 and the results for ETL and MNIST datasets in Figure \ref{Fig: results of 2nd exp}.
\end{remark}

\section{Preliminaries} \label{Sec: preliminaries}

\subsection{Graphs and Laplacians} \label{Subsec: graph and Laplacian}

Let $G = (V, E)$ be an undirected graph, where 
$V$ is the set of $n$ nodes $1, \ldots, n$ and $E$ is the set of edges.
We put a weight on each pair of nodes through the function $w: V \times V \rightarrow \Real_{+}$.
Here, the symbol $\Real_{+}$ denotes the set of nonnegative real numbers.
The function $w$ should have the following properties.
For any pair of nodes $u, v \in V$,
$w(u, v) = w(v, u)$  and  $w(u, v) > 0$ if $\{u, v\} \in E$; otherwise, $w(u, v) = 0$.
We call a function $w$ having the properties above a \emph{weight function} on $G$.
The degree $d_u$ of node $u \in V$ is given as $d_u = \sum_{v \in V} w(u,v)$.
Throughout this paper, we always regard a graph $G$
as an undirected one with $n$ nodes $1, \ldots, n$ and a weight function $w$
and assume that every node of $G$ has a positive degree.

Let us explain the notation, symbols, and terminology that will be used in this paper.
Let $G=(V,E)$ be a graph.
A cluster $S$ in $G$ is a subset of the node set $V$.
A $k$-way partition of $G$ is a family of $k$ clusters $S_1, \ldots, S_k$
that satisfy $S_i \cap S_j = \emptyset$ for different $i$ and $j$ and $S_1 \cup \cdots \cup S_k = V$.
For simplicity, we sometimes call it a partition of $G$ or a graph partition.
The symbol $\Gamma$ is used to denote a $k$-way partition $\{S_1, \ldots, S_k\}$ of $G$. 
The symbol $n_i$ is used to denote the number of nodes in $S_i$.
Let $\{S_1, \ldots, S_k\}$ be a partition of $G$ and 
$u$ be the node such that it belongs to $S_i$
and it has the $j$th smallest degree among all nodes in $S_i$.
We use $(i,j)$ to refer to the node $u$ and
$d_{i,j}$ to refer to the degree $d_u$.
Note that the notation depends on the choice of $k$-way partition $\Gamma$.
Following the above notation, $n_i$ nodes in $S_i$ are expressed as $(i,1), \ldots, (i,n_i)$,
and the degree $d_{i,j}$ of each node $(i,j)$ satisfies $d_{i,1} \le \cdots \le d_{i,n_i}$.
The node $(i, n_i)$ belongs to $S_i$ and has the largest degree
among all nodes in $S_i$.
We call $(i, n_i)$ the \emph{representative node} of the cluster $S_i$ 
and the set $\{ (1,n_1), \ldots, (k, n_k) \}$
the \emph{representative node set} of the $k$-way partition $\Gamma = \{S_1, \ldots, S_k\}$.

Next let us review some basic results from spectral graph theory.
The adjacency matrix $\W$ is an \by{n}{n} symmetric matrix such that 
the $(u, v)$th entry stores the weight $w(u, v)$ of the pair of nodes $u, v \in V$.
The degree matrix $\D$
is an \by{n}{n} diagonal matrix such that the $(u, u)$th entry stores the degree $d_u$ of node $u \in V$.
The \emph{Laplacian} $\L$ of $G$ is given as $\L = \D - \W$,
and the \emph{normalized Laplacian} $\LC$ is given as 
$\LC = \D^{-1/2} \L \D^{-1/2}$, which is equivalent to $\I - \D^{-1/2} \W \D^{-1/2}$.
The eigenvalues and eigenvectors of the normalized Laplacian $\LC$
will play an important role in our discussion.
Since $\LC$ is an \by{n}{n} real symmetric matrix,
the $n$ eigenvalues are real and 
the $n$ eigenvectors can be chosen to be orthonormal bases in $\Real^n$.
Furthermore, an easy calculation shows that $\LC$ is positive semidefinite.
Hence, all the eigenvalues are nonnegative.
The smallest eigenvalue is zero,
since $\LC \cdot (\D^{1/2}\one) = \zero$,
where the symbol $\one$ denotes a vector of all ones.
In addition, the largest eigenvalue is less than two.
The multiplicity of the zero eigenvalue 
equals to the number of connected components of $G$.
The above are basic results from spectral graph theory;
for details, see \cite{Chu97, Lux07}.
In this paper, we will always use the symbols $\lambda_1, \ldots, \lambda_n$
to denote the eigenvalues of $\LC$
arranged in nondecreasing order, i.e., $0 = \lambda_1 \le \cdots \le \lambda_n \le 2$.
Moreover, we will always choose the eigenvectors of $\LC$ to be orthonormal
and use the symbol $\f_i$ to denote the eigenvector
corresponding to the $i$th smallest eigenvalue $\lambda_i$.

\subsection{Conductance} \label{Subsec: conductance}
Let $G = (V,E)$ be a graph. Let $S$ be a cluster in $G$. 
The \emph{conductance of a cluster $S$} is defined to be
\begin{equation} \label{Exp: conductance of cluster}
 \phi(S) :=  \frac{w(S, V \setminus S)}{\mu(S)}
\end{equation}
by letting
\begin{equation*}
 \mu(S) := \sum_{u \in S} d_{u} \quad \mbox{and} \quad
  w(S, V \setminus S) := \sum_{u \in S} \sum_{v \in V \setminus S}  w(u, v).
\end{equation*}
Here, $\mu(S)$ is the volume of $S$, and
$w(S, V \setminus S)$ is the cut size between $S$ and its complement $V \setminus S$.
We can see from the definition of $\phi(S)$ that
clusters with low conductance capture the notion of good clusters in $G$,
wherein nodes within the same cluster have high weights and
nodes within different clusters have low weights.
Kannan et al.\ \cite{Kan04} suggested that 
conductance is an effective way of quantifying the quality of clusters in $G$.

The conductance problem asks
one to find a $k$-way partition $\{S_1, \ldots, S_k\}$ of $G$ 
that minimizes the maximum of $\phi(S_1), \ldots, \phi(S_k)$.
The \emph{$k$-way conductance of a graph $G$} is defined to be the minimum value,
and we use the symbol $\phi_k(G)$ to denote it.
That is, 
\begin{equation*}
 \phi_k(G) = \min_{\{S_1, \ldots, S_k\}} \max \{ \phi(S_1), \ldots, \phi(S_k)\}
\end{equation*}
and the minimum is taken over all candidates of $k$-way partitions of $G$.
We say that a $k$-way partition $\{S_1, \ldots, S_k\}$ of $G$ is \emph{optimal} 
if it satisfies $\phi_k(G) = \max \{ \phi(S_1), \ldots, \phi(S_k)\}$.

Finding an optimal $k$-way partition of $G$ is intractable;
it is known to be NP-hard even if $k=2$; see \cite{Mat90}.
There are approximation algorithms for $k=2$,
and in particular, the SDP-based algorithm of Arora et al.\ in \cite{Aro09}
achieves an $O(\sqrt{\log n})$-approximation ratio.
Cheeger inequality bounds $\phi_2(G)$ by using the second smallest eigenvalue $\lambda_2$
of the normalized Laplacian $\LC$ of $G$.
The bound was improved by Kwok et al.\  \cite{Kwo13}.
Regarding the general case,
Lee et al.\ \cite{Lee12} developed a higher-order Cheeger inequality.
It bounds $\phi_k(G)$ by using the $k$th smallest eigenvalue $\lambda_k$ of $\LC$.

Peng et al.\ \cite{Pen17} examined the performance of KSC for a class of graphs, called well-clustered graphs. 
For a graph $G$,  we define
\begin{equation*}
 \Upsilon := \frac{\lambda_{k+1}}{\phi_k(G)}
\end{equation*}
for the $k$-way conductance $\phi_k(G)$ and 
the $(k+1)$th smallest eigenvalue $\lambda_{k+1}$ of the normalized Laplacian $\LC$.
A graph $G$ is called \emph{well-clustered} if $\Upsilon$ is large.
Let us see why it is well-clustered in that case.
The higher-order Cheeger inequality implies that, 
if $\Upsilon$ is large, so is $\lambda_{k+1} / \lambda_k$.
We recall the result of Gharan and Trevisan \cite{Gha14} who studied graphs with a  gap between $\lambda_k$ and $\lambda_{k+1}$.
Let $S$ be a subset of node set in $G$, i.e., cluster.
The \emph{outside conductance of $S$} is defined to be $\phi(S)$.
The \emph{inside conductance of $S$} is defined to be $\phi_2(G[S])$, 
which is the two-way conductance of a subgraph $G[S]$ induced by $S$.
Let $\Gamma = \{S_1, \ldots, S_k\}$ be a $k$-way partition of $G$.
Following the terminology of Gharan and Trevisan,
we say that $\Gamma$ is a \emph{$(\phiIn, \phiOut)$-clustering}
if $\phi_2(G[S_i]) \ge \phiIn$ and $\phi(S_i) \le \phiOut$ for $i=1, \ldots, k$.
They showed in Corollary 1.1 of \cite{Gha14} that, 
if there is a large gap between $\lambda_k$ and $\lambda_{k+1}$, 
there is a $( \Omega( \lambda_{k+1}  / k  ),    O( k^3 \sqrt{\lambda_k}) )$-clustering.
Thus, we could say that $G$ is well-clustered if $\Upsilon$ is large.
A similar observation can be found in Section 1 of \cite{Kol16}.

Kolev and Mehlhorn \cite{Kol16} used a measurement $\Psi$ different from $\Upsilon$ for analyzing the performance of KSC.
We denote by $U$ be the set of all optimal $k$-way partitions of $G$.
Let 
\begin{equation*}
 \bar{\phi}_k(G) = \min_{\{S_1, \ldots, S_k\}\in U} \frac{1}{k} \left( \phi(S_1) + \cdots + \phi(S_k) \right).
\end{equation*}
In analogy with $\Upsilon$, define
\begin{equation*}
 \Psi := \frac{\lambda_{k+1}}{\bar{\phi}_k(G)}.
\end{equation*}
Since $\bar{\phi}_k(G) \le \phi_k(G)$, we have $\Psi \ge \Upsilon$.

\subsection{Spectral Clustering Algorithm} \label{Subsec: spectral clustering algorithm}
Here, we describe the framework of the spectral clustering algorithm.
The input is the normalized Laplacian $\LC$ of a graph $G = (V, E)$ 
and the number $k$ of clusters the user desires.

\begin{enumerate}
 \item (Embedding stage)~Compute the bottom $k$ eigenvectors $\f_1, \ldots, \f_k \in \Real^n$ of $\LC$ 
       and construct the spectral embedding map $F : V \rightarrow \Real^k$ using them.
       Apply $F$ to the nodes $1, \ldots, n$ of $G$ and form a set $X$ of pints
       $F(1), \ldots, F(n) \in \Real^k$.

 \item (Grouping stage)~Find a $k$-way partition $\{X_1, \ldots, X_k\}$ of $X $ 
       using a clustering algorithm the user prefers.
       Return $\{T_1, \ldots, T_k\}$ by letting $T_i = \{u : F(u) \in X_i\}$ for $i=1, \ldots, k$.

\end{enumerate}
The embedding stage constructs a spectral embedding map, which is defined as follows.
Let $\P = [\f_1, \ldots, \f_k]^\trans \in \Real^{k \times n}$ 
for the bottom $k$ eigenvectors $\f_1, \ldots, \f_k$ of $\LC$, 
and $\p_u$ denote the $u$th column of $\P$.
\emph{Spectral embedding map} is a map $F : V \rightarrow \Real^k$ defined by 
\begin{equation} \label{Exp: spectral embedding map}
 F(u) = s_u \cdot \p_u
\end{equation}
for a scaling factor $s_u \in \Real$.
The scaling factor is often set as $s_u = 1 / \sqrt{d_u}$ for the degree $d_u$ of node $u$ or 
$s_u = 1 / \|\p_u\|_2$.
The former was proposed by Shi and Malik \cite{Shi00} and the latter by Ng et al.\ \cite{Ng02}.

It is standard practice to use the $k$-means method based on Lloyd's algorithm \cite{Llo82} 
in the grouping stage (this was suggested in \cite{Ng02, Lux07}).
We quickly review the $k$-means method here.
Let $\p_1, \ldots, \p_n$ be points in $\Real^k$.
We arbitrarily choose a $k$-way partition $\{S_1, \ldots, S_k\}$ of the set $S = \{\p_1, \ldots, \p_n\}$.
As in the case of a $k$-way partition of a graph, we use the symbol $\Gamma$ to refer to it.
The clustering cost function $f$ is given by 
\begin{equation} \label{Exp: clustering cost function}
 f(\Gamma) = \min_{\c_1, \ldots, \c_k \in \Real^k}
  \sum_{i=1}^{k} \sum_{u \in S_i} \|\p_u - \c_i \|_2^2.
\end{equation}
The $k$-means method chooses a $k$-way partition $\Gamma$ of $S$
to minimize the clustering cost function $f(\Gamma)$.
Finding $\Gamma$ that minimizes $f$ is shown to be NP-hard in \cite{Alo09, Mah09}.
Lloyd's algorithm approximately solves the minimization problem.
It starts by arbitrarily choosing $\c_1, \ldots, \c_k$ 
as initial seeds and then minimizes $f(\Gamma)$ by alternatively fixing either $\c_1, \ldots, \c_k$ or $S_1, \ldots, S_k$.
This works well in practice.
The $k$-means++ algorithm presented in \cite{Art07} provides a smart choice of initial seeds.
It chooses $\c_1$ uniformly at random from the set of data points
and then chooses $\c_{i+1}$ from the set according to a probability determined 
by the choice of $\c_1, \ldots, \c_i$.

\section{Algorithm and Analysis Results}\label{Sec: proposed algorithm}

\subsection{Outline of ELLI} \label{Subsec: outline}
Peng et al.\ developed the structure theorem (Theorem 3.1 of \cite{Pen17}) for analyzing the performance of KSC.
Later, the theorem was extended in \cite{Miz20}.
ELLI is built upon the extension of the structure theorem.
Let us recall it here.
For a $k$-way partition $\{S_1, \ldots, S_k\}$ of a graph, 
we define the {\it indicator} $\g_i \in \Real^n$ of $S_i$
to be the vector whose $j$th element is one if $j \in S_i$ and zero otherwise.
The {\it normalized indicator} $\bar{\g}_i \in \Real^n$ of $S_i$ is given as 
\begin{equation*}
 \bar{\g}_i = \frac{\D^{1/2}\g_i}{\| \D^{1/2}\g_i \|_2}
\end{equation*}
for the degree matrix $\D$ of the graph.
Note that $\| \D^{1/2}\g_i \|_2$ is equal to $\sqrt{\mu(S_i)}$.

\begin{theo}[Corollary 1 of \cite{Miz20}] \label{Theo: structure theorem}
 Let a graph $G$ satisfy $\Upsilon > 0$.
 Let a $k$-way partition $\{S_1, \ldots, S_k\}$ of $G$ be optimal.
 Form $\bar{\G} = [\bar{\g}_1, \ldots, \bar{\g}_k] \in \Real^{n \times k}$ 
 for the normalized indicators $\bar{\g}_1, \ldots, \bar{\g}_k$ of $S_1, \ldots, S_k$,
 and form $\F = [\f_1, \ldots, \f_k] \in \Real^{n \times k}$ 
 for the bottom $k$ eigenvectors $\f_1, \ldots, \f_k$ of the normalized Laplacian of $G$.
 Then, there is some \by{k}{k} orthogonal matrix $\U$ such that 
 \begin{equation*}
  \| \F \U - \bar{\G} \|_2   \le  k / \Upsilon +  \sqrt{k / \Upsilon}.
 \end{equation*}
\end{theo}
Since the relation $\sqrt{x} \ge x$ holds for $0 \le x \le 1$,
the corollary implies that $\| \F \U - \bar{\G} \|_2   \le  2 \sqrt{k / \Upsilon}$ if $\Upsilon \ge k$.
A spectral clustering algorithm maps the nodes $1, \ldots, n$ of a graph 
onto the points $F(1), \ldots, F(n)$ in $\Real^k$ by using the spectral embedding map $F$ 
and then partitions $F(1), \ldots, F(n)$ into $k$ groups.
Theorem \ref{Theo: structure theorem} tells us how  $F(1), \ldots, F(n)$ are located in $\Real^k$.
Let $\P = \F^\trans \in \Real^{k \times n}$, and $\p_u$ denote the $u$th column of $\P$.
In the same way, let $\Q = \bar{\G}^\trans \in \Real^{k \times n}$, 
and $\q_u$ denote the $u$th column of $\Q$.
We choose the spectral embedding map defined by $F(u) = \p_u$.
The theorem implies that $\P$ and $\Q$ are related as follows:
\begin{equation*} 
 \P = \U \Q + \R
\end{equation*}
where 
$\U$ is a \by{k}{k} orthogonal matrix and 
$\R$ is a \by{k}{n} matrix that satisfies $\|\R\|_2 \le   2 \sqrt{k / \Upsilon}$ if $\Upsilon \ge k$.
Our choice of $F$ satisfies $F(u) = \p_u$.
The relationship shown above tells us that $\p_u$ is close to $\U\q_u$  
for $u = 1, \ldots, n$ if $\Upsilon$ is large.

Let us look closely at the columns of $\P$ and $\Q$.
The matrix $\Q$ is the transpose of
$\bar{\G} = [\bar{\g}_1, \ldots, \bar{\g}_k]$
where $\bar{\g}_i = \D^{1/2} \g_i /  \|\D^{1/2} \g_i\|_2$
for the indicator $\g_i$ of $S_i$ and the degree matrix $\D$ of $G$.
Hence, if the node $u$ belongs to $S_i$, the $u$th column $\q_u$ of $\Q$ is 
\begin{equation*}
\q_u  = \sqrt{\frac{d_u}{\mu(S_i)}} \e_i.
\end{equation*}
Here, recall the notation for describing the nodes of a graph
that we introduced in Section \ref{Subsec: graph and Laplacian}.
Let $\{S_1, \ldots, S_k\}$ be a $k$-way partition of a graph
and $u$ be the node such that it belongs to $S_i$ and 
has the $j$th smallest degree among all nodes in $S_i$.
The notation $(i,j)$ refers to the node $u$.
Following this notation, 
let $\p_{i,j}$, $\q_{i,j}$ and $\r_{i,j}$ denote the $u$th columns
$\p_u$, $\q_u$, and $\r_u$ of $\P$, $\Q$, and $\R$, respectively.
The columns of $\P$ and $\Q$ can be expressed as 
\begin{equation} \label{Eq: expression of p and q}
 \p_{i,j} = \alpha_{i,j} \u_i  + \r_{i,j} \quad \mbox{and} \quad
  \q_{i,j} = \alpha_{i,j} \e_i
\end{equation}
by using $\alpha_{i,j}$ defined by
\begin{equation} \label{Exp: alpha}
 \alpha_{i,j} := \sqrt{\frac{d_{i,j}}{\mu(S_i)}}
\end{equation}
for $i=1,\ldots,k$ and $j = 1, \ldots, n_i$.
Here, $\u_i$ denotes the $i$th column of the orthogonal matrix $\U$.
Let $\Upsilon \ge k$.
Since $\|\r_{i,j} \|_2 \le \|\R\|_2 \le 2 \sqrt{k / \Upsilon}$ holds,
the distance between the point $\p_{i,j}$  and the line spanned by $\u_i$ is at most $2 \sqrt{k / \Upsilon}$.
Hence, if $\Upsilon$ is large, then $\p_{i,j}$ is close to the line spanned by $\u_i$
that is orthogonal to $\u_\ell$ for $\ell \neq i$.
Figure \ref{Fig: Spectral embedding} illustrates the columns of $\P$.

\begin{figure}[h]
 \centering
 \includegraphics[width=0.85\linewidth]{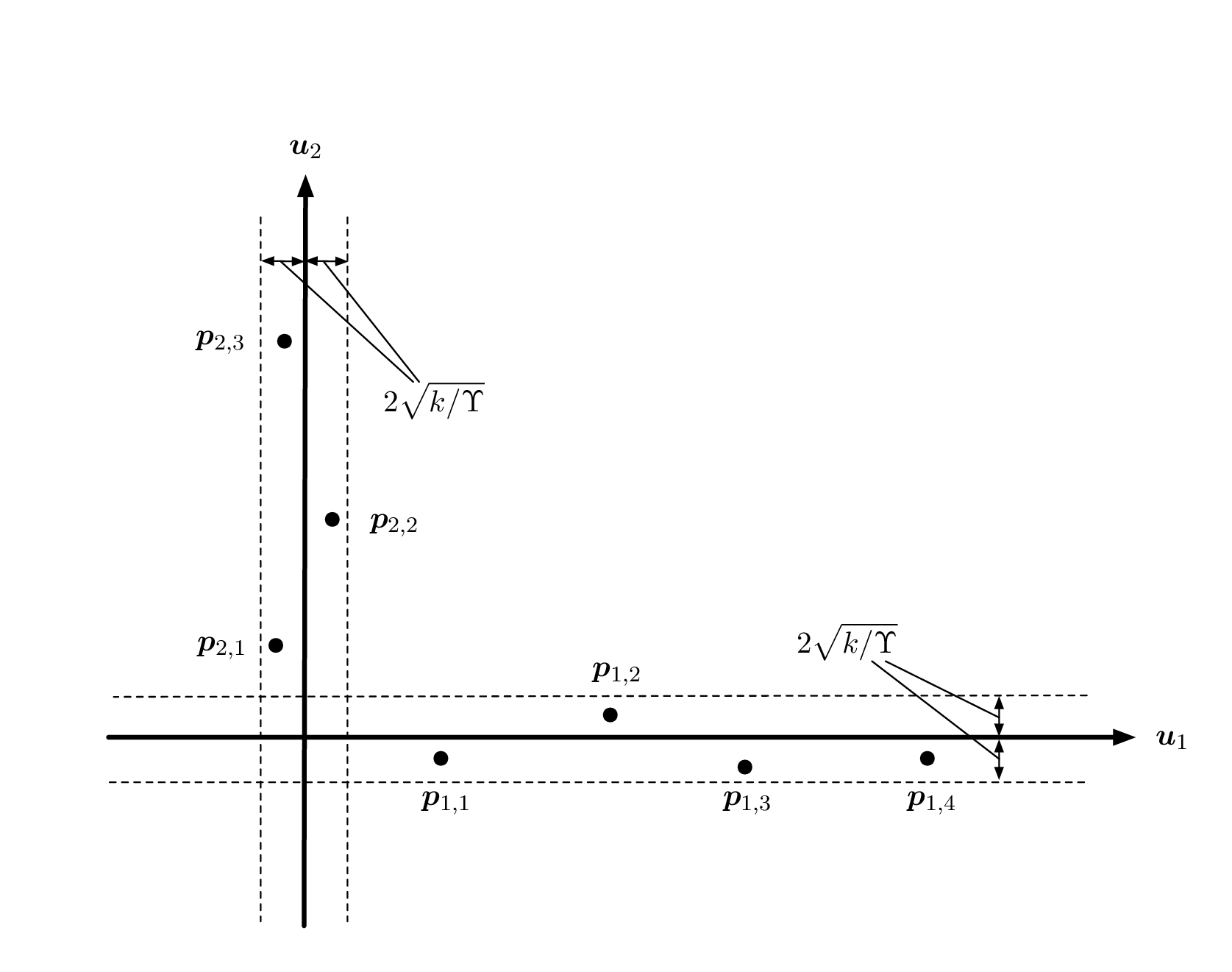}
 \caption{Illustration of the columns $\p_{i,j}$ of $\P$ in the case of $k=2$.}
 \label{Fig: Spectral embedding}
\end{figure}

The task in the grouping stage is to partition the columns $\p_1, \ldots, \p_n$ of $\P$ into $k$ groups.
In particular, the goal is to find $k$ groups that correspond to $k$ clusters $S_1, \ldots, S_k$ 
in an optimal $k$-way partition of a graph.
Based on the observations described above, we develop an algorithm for this task.
In particular, we focus on the orthogonality of $\u_1, \ldots, \u_k$.
Assume that we have exactly one element for each $S_1, \ldots, S_k$.
Let $u_i$ denote the element of $S_i$ that we have
and  $I$ be the set of $u_1, \ldots, u_k$.
Our strategy is as follows.
Initialize  sets $T_1, \ldots, T_k$ to be empty,
and repeat the following procedure from $v=1$ until $n$:
find 
\begin{equation*}
 i^* = \arg \max_{i = 1, \ldots, k} \p_{u_i}^\trans \p_v 
\end{equation*}
for column $\p_v$, and store the column index $v$ in $T_{i^*}$.
The obtained $T_i$ coincides with $S_i$ for $i=1, \ldots, k$ if $\R = \zero$.
This is because
the value of $\p_{u_i}^\trans \p_v$ for $v \in S_j$
is positive if $i = j$; otherwise, it is zero,
and hence, $v$ is stored in $T_j$.
In Section \ref{Subsec: analysis of Step 3}, we show in Theorem \ref{Theo: cluster extraction} 
that $T_i$ still coincides with $S_i$ for $i=1, \ldots, k$ 
if $\| \R \|_2$ is smaller than some threshold.

In the implementation of this strategy,
a question arises as to how to find the set $I$ that
contains $k$ elements belonging to each of $S_1, \ldots, S_k$.
To address this question, we leverage the ellipsoidal rounding (ER) algorithm in \cite{Miz14}.
That is, we compute the minimum-volume ellipsoid, centered at the origin, 
that contains the columns of $\P$ and then find all points on the boundary of the ellipsoid.
ER was originally developed for solving separable NMF problems 
whose formal description is given in Section \ref{Sec: connection with separable NMFs}.
In Section \ref{Subsec: analysis of Step 2}, 
we show in Theorem \ref{Theo: finding cluster representatives} that
the set of obtained points exactly coincides with the representative node set of $\{S_1, \ldots, S_k\}$
if $\|\R\|_2$ is smaller than some threshold.

\subsection{Description of the Algorithm} \label{Subsec: algorithm description}
We describe each step of ELLI in Algorithm \ref{Alg: ELLI}.
Step 1 is the embedding stage, and Steps 2 and 3 are the grouping stage.
Step 2 aims to find the representative node set of a partition of a graph $G$.
Step 3 aims to find the partition of $G$ by using the representative node set.

Step 2 should be explained in detail.
The ellipsoid centered at the origin in $\Real^k$ is
a set $H = \{ \a \in \Real^k : \a^\trans \M \a \le 1\}$
for a \by{k}{k} symmetric positive definite matrix $\M$.
The volume of $H$ is $ v(k) / \sqrt{\det \M}$, 
where $v(k)$ is the volume  of a unit ball in $\Real^k$,
and the value depends on  $k$.
Step 2 constructs a minimum-volume enclosing ellipsoid (MVEE) centered at the origin
for the set $X$ of points $\p_1, \ldots, \p_n$ in $\Real^k$.
The ellipsoid can be obtained by solving an optimization problem
with a symmetric matrix variable $\X$,
\begin{equation*}
\begin{array}{lll}
 \Prob(S): & \mbox{minimize}    &  - \log \det \X, \\
           &  \mbox{subject to} & \p^\trans \X \p \le 1 \ \mbox{for all} \ \p \in S, \\
           &                    & \X \succ \zero.  
\end{array}
\end{equation*}
The notation $\X \succ \zero$ means that $\X$ is positive definite.
The origin-centered MVEE for $S$, denoted by $H(S)$, 
is given as $H(S) = \{ \a \in \Real^k : \a^\trans \X \a \le 1\}$
for the optimal solution $\X$ of $\Prob(S)$.
We call a point $\p_i$ and its index $i$ 
the \emph{active point} and the \emph{active index} of $H(S)$,
if $\p_i$ satisfies $\p_i^\trans \X \p_i = 1$;
in other words, $\p_i$ lies on the boundary of $H(S)$.
Step 2 may use the successive projection algorithm (SPA) in \cite{Ara01, Gil14a} 
that is usually used for solving separable NMF problems.
In Section \ref{Sec: connection with separable NMFs}, we explain the connection between 
finding the representative node set of a partition of $G$ and solving separable NMF problems.

\begin{algorithm}
 \caption{ELLI: Convex programming based spectral clustering}
 \label{Alg: ELLI}
 \smallskip
 Input: $\LC$, the normalized Laplacian of a graph;
 and $k$, the  desired number of clusters. \\
 Output: $\{T_1, \ldots, T_k\}$.
 \begin{enumerate}[1.]
  \item Compute the bottom $k$ eigenvectors $\f_1, \ldots, \f_k$ of $\LC$.
	Let $\P = [\f_1, \ldots, \f_k]^\trans \in \Real^{k \times n}$ and 
	form the set $X = \{\p_1, \ldots, \p_n\}$ for the columns $\p_1, \ldots, \p_n$ of $\P$.

  \item Compute the minimum-volume enclosing ellipsoid $H(S)$ centered at the origin
	for the set $X$, and construct the set $I$ of active indices of $H(S)$.
	If $|I| > k$, 
	choose $k$ elements from $I$ by using the successive projection algorithm,
	and update $I$ by storing the chosen $k$ elements in it.

  \item	Let $u_1, \ldots, u_k$ denote the elements of $I$.
	Set $\bar{\p}_i  = \p_i / \| \p_i \|_2$ for $i = 1, \ldots, n$.
	Initialize the sets $T_1, \ldots, T_k$ to be empty,
	and repeat the following procedure from $v = 1$ until $n$.
	\begin{enumerate}[-]
	 \item Pick $\bar{\p}_v$ and
	       find $\displaystyle i^* = \arg \max_{i = 1, \ldots, k} \bar{\p}_{u_i}^\trans \bar{\p}_v$;
	       if multiple indices achieve the maximum, choose one of them.
	 \item Update $T_{i^*}$ to be $T_{i^*} \cup \{v\}$.
	\end{enumerate}
	Then, return the family of sets $T_1, \ldots, T_k$.
 \end{enumerate}
\end{algorithm}

Let us examine the computational cost of Steps 2 and 3
(as Step 1 is a common to spectral clustering algorithms).
The main cost in Step 2 is in computing the optimal solution of problem $\Prob(S)$.
This is a convex programming problem, and efficient algorithms exist for solving it.
Khachiyan \cite{Kha96} developed the Frank-Wolfe algorithm for solving the dual problem
and evaluated the computational cost.
Kumar and Yildirim  \cite{Kum05} modified the algorithm and 
showed that the modification returns a $(1+\epsilon)$-approximation solution in $O(nk^3 / \epsilon)$.
An interior-point method within a cutting plane framework
can quickly solve the problem in practice.
The main cost in Step 3 is in computing $\bar{\p}_{u_i}^\trans \bar{\p}_v$
for $i=1,\ldots,k$ and $v=1,\ldots,n$. The computation takes $O(nk^2)$.

The author presented a spectral clustering algorithm in \cite{Miz15}.
This algorithm shares Steps 1 and 2 in common with ELLI, but does not share Step 3.
That manuscript mainly studied
the similarity between algorithms for spectral clustering and separable NMFs.

\subsection{Results of Performance Analysis} \label{Subsec: results of performance analysis}
Here, we state the results of our analysis;
the details are given in Section \ref{Sec: analysis}. 
Let $\alpha_{i,j}$ be defined as in (\ref{Exp: alpha}).
They satisfy $\alpha_{i,1} \le \cdots \le \alpha_{i,n_i}$.
Define
\begin{equation*}
\alpha_{\mmin} := \min_{\substack{i=1, \ldots, k \\ j=1, \ldots, n_i}} \alpha_{i,j}
 \quad \mbox{and} \quad
 \hat{\alpha}_{\mmin} := \min_{i=1, \ldots, k} \alpha_{i, n_i}.
\end{equation*}
We can rewrite $\alpha_{\mmin}$ as $\alpha_{\mmin} = \min_{i=1, \ldots, k} \alpha_{i,1}$.
Define 
\begin{equation} \label{Exp: theta}
 \theta_{i,j} := \frac{\alpha_{i,j}}{\alpha_{i,n_i}}
  = \sqrt{\frac{d_{i,j}}{d_{i, n_i}}}
\end{equation}
for $i = 1, \ldots, k$ and $j = 1, \ldots, n_i-1$.
As is the case with $\alpha_{i,j}$,
they satisfy $\theta_{i,1} \le \cdots \le \theta_{i, n_i-1}$.
Define 
\begin{equation*}
 \theta_{\mmin} := \min_{\substack{i=1,\ldots,k \\ j=1,\ldots,n_i-1}} \theta_{i,j}
 \quad \mbox{and} \quad
 \theta_{\mmax} := \max_{\substack{i=1,\ldots,k \\ j=1, \ldots, n_i-1}} \theta_{i,j}.
\end{equation*}
We can rewrite them as $\theta_{\mmin} = \min_{i=1,\ldots,k}\theta_{i,1}$ and $\theta_{\mmax} = \max_{i=1,\ldots,k}\theta_{i,n_i-1}$.
Now let us introduce the parameters $\alpha$ and $\theta$.
These parameters are determined by $\alpha_{i,j}$ and $\theta_{i,j}$,
which are determined by choosing one of the $k$-way partitions of a graph.
Let a $k$-way partition $\Gamma$ be optimal.
The parameter $\alpha$ is set as 
\begin{equation*}
 \alpha = \hat{\alpha}_{\mmin}
\end{equation*}
for $\hat{\alpha}_{\mmin}$ determined by the optimal $k$-way partition $\Gamma$.
This satisfies $\alpha > 0$, since $d_{i,j}$ are all positive.
The parameter $\theta$ is set as
\begin{equation*}
 \theta = \min \biggl\{\frac{1}{2} (1 - \theta_{\mmax}), \ (17-12\sqrt{2})\theta_{\mmin} \biggr\}
\end{equation*}
for $\theta_{\mmin}$ and $\theta_{\mmax}$ determined by the optimal $k$-way partition $\Gamma$.
This satisfies $\theta \ge 0$, 
since $0 < \theta_{\mmin}\le \theta_{\mmax} \le 1$.
Here, $\theta_{\mmin}$ is strictly greater than zero,
as  $d_{i,j}$ are all positive.
Theorem \ref{Theo: main result} is the main result of our analysis.

\begin{theo} \label{Theo: main result}
 If a graph $G$ satisfies
 \begin{equation*}
  \Upsilon  > \frac{4k}{(\theta \alpha)^2},
 \end{equation*}
 then the output of ELLI coincides with an optimal $k$-way partition of $G$.
\end{theo}
The proof is given in Section \ref{Subsec: proof of our main results}.
Here, let us compare Theorem \ref{Theo: main result} with the results of Peng et al.\  \cite{Pen17} and Kolev and Mehlhorn \cite{Kol16} 
(a detailed description of their results is given in Section \ref{Subsec: related work}).
The results of Peng et al.\
tell us that the output of KSC gets closer to the optimal $k$-way partition of a graph as $\Upsilon$ gets larger.
However, it does not ensure that the output is exactly the optimal one no matter how large $\Upsilon$ is.
The same goes for the results of Kolev and Mehlhorn.
Meanwhile, Theorem \ref{Theo: main result} ensures that
the output of ELLI is the optimal one if $\Upsilon$ exceeds some threshold.
But, the threshold could be large.
Let us look at this in more detail.
Peng et al.\ showed the performance of KSC for graphs satisfying $\Upsilon = \Omega(k^3)$.
Let us compare the bound on $\Upsilon$ in Theorem \ref{Theo: main result} with that of Peng et al.
Although the bound of Theorem \ref{Theo: main result} contains $\theta$ and $\alpha$ as well as $k$,
we can rewrite it  using only $k$ as in the following corollary.
\begin{coro} \label{Coro: obtained from main result}
 Let positive real numbers $p, q, r$ satisfy the following three conditions.
 For every $ i=1, \ldots, k $,
 \begin{itemize}
  \item $\displaystyle d_{i, n_i-1} \le \left( 1 - \frac{1}{k^{p/2}} \right)^2 \cdot d_{i, n_i}$,
  \item $\displaystyle d_{i,1} \ge \left( \frac{1}{k^q} \right) \cdot d_{i, n_i}$,
  \item $\displaystyle d_{i,n_i} \ge \left( \frac{1}{k^r} \right) \cdot \mu(S_i)$.
\end{itemize}
 If a graph $G$ satisfies $\Upsilon = \Omega \left( k^{(r+1) \cdot \max\{p, q\}} \right)$, 
 then  the output of ELLI coincides with an optimal $k$-way partition of $G$.
\end{coro}
The proof is given in Section \ref{Subsec: proof of our main results}.
We can  bound $\mu(S_i)$ from above by $n_i \cdot d_{i, n_i}$.
This  bound implies that the choice of $r = \log_k n_i$ 
satisfies the third condition of Corollary \ref{Coro: obtained from main result}.
Hence, if $n_i \ge k^2$ and $\max\{p,q\} \ge 1$, we have  $(r+1) \cdot \max\{p, q\} \ge 3$.
Accordingly, in most cases, Theorem \ref{Theo: main result} imposes a stronger restriction on graphs
compared with the results of Peng et al.

\subsection{Related Work} \label{Subsec: related work}
We  describe the results of Peng et al.\ \cite{Pen17} in more detail.
Let a $k$-way partition $\{S_1, \ldots, S_k\}$ of a graph be optimal.
We choose an algorithm for minimizing the clustering cost function $f$ shown in (\ref{Exp: clustering cost function})
and assume that the algorithm has an approximation ratio of $\eta$.
Let $T_1, \ldots, T_k$ be the output of a spectral clustering algorithm that
uses a $k$-means method based on the $\eta$-approximation algorithm.
Peng et al.\ showed in Theorem 1.2 of \cite{Pen17} that,  
if a graph satisfies $\Upsilon = \Omega(k^3)$,
then the following holds after a suitable renumbering of the output of the algorithm.
\begin{equation*}
 \mu(S_i \triangle T_i) = O\biggl(\frac{\eta k^3}{\Upsilon} \biggr) \mu(S_i)
  \quad \mbox{and} \quad
  \phi(T_i) = 1.1 \phi(S_i) + O\biggl(\frac{\eta k^3}{\Upsilon} \biggr). 
\end{equation*}
Here, the notation $S_i \triangle T_i$ denotes the symmetric difference of the sets $S_i$ and $T_i$,
i.e., $S_i \triangle T_i = (S_i \setminus T_i) \cup (T_i \setminus S_i)$.
This result tells us that, if $\Upsilon = \Omega(k^3)$,
the output $\{T_1, \ldots, T_k\}$ is close to the optimal $k$-way partition $\{S_1, \ldots, S_k\}$ of the graph, 
and, in particular, the output gets closer to the optimal one as $\Upsilon$ gets larger.
In \cite{Pen17}, Peng et al.\ also developed a nearly linear time algorithm for clustering  
by using the heat kernel of a graph and  nearest neighbor data structures.

Kolev and Mehlhorn \cite{Kol16} improved on the results of Peng et al.
They showed in Theorem 1.2 of their paper that, 
if a graph satisfies $\Psi = \Omega(k^3)$, 
then the following holds after a suitable renumbering of the output of the algorithm.
\begin{equation*}
 \mu(S_i \triangle T_i) = O\biggl(\frac{\eta k^2}{\Psi} \biggr) \mu(S_i)
  \quad \mbox{and} \quad
 \phi(T_i) = 1.1 \phi(S_i) + O\biggl(\frac{\eta k^2}{\Psi} \biggr). 
\end{equation*}
The results of Kolev and Mehlhorn are an improvement over 
those of Peng et al.
The bounds on the approximation accuracy are reduced by a factor of $k$ 
and the gap assumption is weakened  due to $\Psi \ge \Upsilon$. 
In \cite{Kol16}, Kolev and Mehlhorn also analyzed a spectral clustering algorithm
that uses a variant of Lloyd's algorithm, which was proposed in \cite{Ost12}, for the $k$-means method.
The results of Kolev and Mehlhorn  were further improved in \cite{Miz20}.

There is a line of research that explores spectral clustering from a theoretical perspective.
Spectral clustering maps the nodes of a graph onto points in  real space through the spectral embedding map.
Using the Davis-Kahan theorem from matrix perturbation theory,
Ng et al.\ \cite{Ng02} showed that the resulting points are nearly orthogonal.
Kannan et al.\  \cite{Kan04} introduced bicriteria to quantify the quality of clusters
where one criterion is the inside conductance of a cluster,
which was explained in Section \ref{Subsec: conductance}, 
and the other is the total weight of the inter-cluster edges.
They assumed that a graph has clusters such that 
the inside conductance of the clusters is large and 
the total weight of the inter-cluster edges is small.
They evaluated how close the output is to the clusters.
As we saw in Section \ref{Subsec: conductance},
Gharan and Trevisan \cite{Gha14} showed that, if there is a large gap between $\lambda_k$ and $\lambda_{k+1}$, 
there exists a $k$-way partition of a graph with a large inside conductance and small outside conductance.
Here, recall that $\lambda_k$ and $\lambda_{k+1}$ are the $k$th and the $(k+1)$th smallest eigenvalues 
of the normalized Laplacian of the graph.
They also showed that, if a graph satisfies $\lambda_{k+1} > 0$, 
there is a polynomial time algorithm 
that outputs a $\ell$-way partition of the graph that is 
a $(\Omega(\lambda_{k+1}^2 / k^4), O(k^6\sqrt{\lambda_k}) )$-clustering where $1 \le  \ell \le k$.
Dey et al.\ \cite{Dey19} developed a greedy algorithm that partitions the node set of a graph 
into clusters with a large inside conductance and small outside conductance.
They showed that, if there is a large gap between $\lambda_k$ and $\lambda_{k+1}$,
the output of the algorithm is close to a $(\Omega(\lambda_{k+1} / k), O(k^3\sqrt{\lambda_k}) )$-clustering.
Sinop \cite{Sin16} studied a spectral clustering algorithm in the context of the edge expansion problem,
which is related to the conductance problem,
and evaluated the accuracy of its output by using a similar measurement to $\Upsilon$.

There is also a considerable amount of research on spectral clustering on a random graph. 
In the planted partition model, one assumes that the node set is partitioned into several clusters
and edges connecting the nodes are stochastically generated:
any two nodes in the same cluster have an edge with probability $p$, and
any two nodes in different clusters have an edge with probability $q$.
McSherry \cite{Mcs01} showed that spectral clustering can identify the clusters
with high probability if $p$ and $q$ lie within some range.
Rohe \cite{Roh11} and Lei \cite{Lei15} studied KSC on a stochastic block model.

\section{Analysis of the Algorithm} \label{Sec: analysis}

\subsection{Step 2} \label{Subsec: analysis of Step 2}
We  analyze Step 2 of ELLI and state the result (Theorem \ref{Theo: finding cluster representatives}).
The analysis of Step 3 and the result (Theorem \ref{Theo: cluster extraction}) are given in the next section.
Let $\P \in \Real^{k \times n}$ be the matrix constructed in Step 1.
We take a $k$-way partition $\Gamma = \{ S_1, \ldots, S_k \}$ of a graph and 
construct $\Q = [\bar{\g}_1, \ldots \bar{\g}_k]^\trans \in \Real^{k \times n}$
for the normalized indicators $\bar{\g}_1, \ldots, \bar{\g}_k$ of $S_1, \ldots, S_k$.
Choosing a \by{k}{k} orthogonal matrix $\U$, we express $\P$ as
\begin{equation} \label{Exp: P}
 \P = \U\Q + \R
\end{equation}
where $\R$ is a \by{k}{n} and serves as the residual between $\P$ and $\U\Q$.
In the analysis of Steps 2 and 3, we use this expression for $\P$. 
As explained in Section \ref{Subsec: outline}, 
if $\Gamma$ is optimal and $\Upsilon \ge k$ holds,
there is some orthogonal matrix $\U$ such that $\R$ satisfies $\|\R\|_2 \le 2 \sqrt{k / \Upsilon}$.
It should be noted that we do not specify $\Gamma$ to be optimal
in Theorems  \ref{Theo: finding cluster representatives} and \ref{Theo: cluster extraction}.
We will use the result presented in \cite{Miz14} for analyzing Step 2.
\begin{prop}[Corollary 4 of \cite{Miz14}] \label{Prop: active points}
 Let $\a_1, \ldots, \a_k, \b_1, \ldots, \b_m$ be points in $\Real^k$, 
 and let $\M = [\a_1, \ldots, \a_k] \in \Real^{k \times k}$.
 Assume that the points satisfy the following conditions.
 \begin{itemize}
  \item $\M$ is nonsingular.
  \item For any $\b \in \{\b_1, \ldots, \b_m\}$,
	there exists some vector $\c \in \Real^k$ such that $\b = \M\c$ and $\|\c\|_2 < 1$.
 \end{itemize}
 Let $H(S)$ be an origin-centered MVEE for the set $S$ of points $\a_1, \ldots, \a_k, \b_1, \ldots, \b_m$.
 Then, the active points of $H(S)$ are $\a_1, \ldots, \a_k$.
\end{prop}
Using the Karush-Kuhn-Tucker (KKT) conditions,
we can easily check the correctness of this assertion.
For the set $S$ of the points $\a_1, \ldots, \a_k, \b_1, \ldots, \b_m$,
these conditions imply that
$(\M\M^\trans)^{-1}$ is an optimal solution for problem $\Prob(S)$.
We have $\a_i^\trans (\M\M^\trans)^{-1} \a_i = 1$ and, for any $\b \in \{\b_1, \ldots, \b_m\}$,
\begin{equation*}
 \b^\trans (\M\M^\trans)^{-1} \b =  \c^\trans \M^\trans (\M^\trans)^{-1} \M^{-1} \M \c
  = \|\c\|_2^2 < 1.
\end{equation*}
Hence, the active points of the origin-centered MVEE for $S$ are $\a_1, \ldots, \a_k$.

\begin{theo} \label{Theo: finding cluster representatives}
 Let $\P = \U\Q + \R$ be of the form shown in (\ref{Exp: P}).
 Let $\Gamma$ be the $k$-way partition of a graph corresponding to $\Q$.
 Let $H(S)$ be an origin-centered MVEE for the set of all columns of $\P$.
 If
 \begin{equation*}
  \| \R \|_2 <  \frac{1}{2} (1-\theta_{\mmax})\hat{\alpha}_{\mmin},
 \end{equation*}
 then the active index set of $H(S)$ coincides with the representative node set of $\Gamma$.
\end{theo}
Recall that $\theta_{i,j}$ and $\alpha_{i,j}$ are defined as in (\ref{Exp: alpha}) and (\ref{Exp: theta}), respectively;
in the theorem above, $\theta_{\mmax}$ is the largest among $\theta_{i,j}$ for $i=1,\ldots, k, \ j=1, \ldots, n_i-1$, 
and $\hat{\alpha}_{\mmin}$ is the smallest among $\alpha_{i, n_i}$ for $i=1, \ldots, k$.
The proof of this theorem relies on the techniques used to prove  Theorem 9 of \cite{Miz14}, 
which gives the robustness of ER to noise. 
However, as we will see in Section \ref{Sec: connection with separable NMFs}, 
Theorem \ref{Theo: finding cluster representatives} does not directly follow from it.
The proof is thus presented.
As a lemma for proving the theorem, 
we use a classical result regarding singular value perturbations.
For a matrix $\A$,
the symbol $\sigma_i(\A)$ denotes the $i$th smallest singular value of $\A$.
In particular, the symbol $\sigma_{\mmin}(\A)$ denotes
the smallest singular value, i.e., $\sigma_1(\A)$.

\begin{lemm}[See, for instance, Corollary 8.6.2 of \cite{Gol13}]
 \label{Lemm: singular value perturbation}
 Let $\A \in \Real^{k \times n}$  and $\N \in \Real^{k \times n}$. 
 We have
 \begin{equation*}
  |\sigma_i(\A + \N) - \sigma_i(\A)| \le \| \N \|_2
 \end{equation*}
 for $i = 1, \ldots, \ell$ where $\ell = \min \{k, n\}$.
\end{lemm}

Let us prove Theorem \ref{Theo: finding cluster representatives}.
\begin{proof}[Proof of Theorem \ref{Theo: finding cluster representatives}]
 We use $\p_{i,j}$ to refer to the columns of $\P$.
 Let $\M = [\p_{1,n_1}, \ldots, \p_{k,n_k}] \in \Real^{k \times k}$ and
 $\p \in \Real^k$ be a vector arbitrarily chosen
 among the vectors $\p_{i,j}$ with $j \notin \{n_1, \ldots, n_k\}$.
 From Proposition \ref{Prop: active points},
 it is sufficient to prove that $\M$ is nonsingular and
 there exists some vector $\c \in \Real^k$ such that $\p = \M\c$ and $\|\c\|_2 < 1$.

 We use $\q_{i,j}$ to refer to the columns of $\Q$.
 Since $\q_{i,j} = \alpha_{i,j} \e_i$ as shown in (\ref{Eq: expression of p and q}),  
 we can write $\M$ as
 \begin{equation*}
  \M = \U \V + \R'
 \end{equation*}
 by letting $\V = \mbox{diag}(\alpha_{1,n_1}, \ldots, \alpha_{k,n_k}) \in \Real^{k \times k}$ and
 $\R' = [\r_{1,n_1}, \ldots, \r_{k, n_k}] \in \Real^{k \times k}$.
 It follows from Lemma \ref{Lemm: singular value perturbation} that
 \begin{equation*}
  |\sigma_{i}(\M) - \sigma_i(\U\V) | = |\sigma_{i}(\M) - \sigma_i(\V) | \le \|\R'\|_2 \le \|\R\|_2  
 \end{equation*}
 for $i=1,\ldots,k$.
 Here, the equality comes from that $\U$ is orthogonal,
 and the second inequality comes from that $\R'$ is a submatrix of $\R$.
 Hence, we have $\sigma_{\mmin}(\M) \ge  \sigma_{\mmin}(\V) - \|\R\|_2 = \hat{\alpha}_{\mmin} - \|\R\|_2$.
 From the bound on $\|\R\|_2$ imposed in this theorem, the inequality implies 
\begin{equation} \label{Exp: bound on the smallest singular value of S}
 \sigma_{\mmin}(\M) > \frac{1}{2} (1 + \theta_{\mmax}) \hat{\alpha}_{\mmin}.
\end{equation}
Since $\hat{\alpha}_{\mmin}$ and $\theta_{\mmax}$ are positive,
so is $\sigma_{\mmin}(\M)$.
Accordingly, $\M$ is nonsingular.

Let $\r = \r_{i,j} - \theta_{i,j} \r_{i, n_i} \in \Real^k$, and
set a vector $\c \in \Real^k$ as $\c = \theta_{i,j} \e_i + \M^{-1} \r$.
We have $\p = \p_{i,j} = \M\c$, since
\begin{align*}
 \M\c  
 & = \M (\theta_{i,j} \e_i + \M^{-1} \r) \\
 & = \alpha_{i,j} \u_i + \theta_{i,j} \r_{i, n_i} + \r \\
 & = \alpha_{i,j} \u_i + \r_{i,j} \\
 & = \p_{i,j}.
\end{align*}
 We can bound $\|\c\|_2$ as 
\begin{equation} \label{Exp: bound on norm of c}
 \|\c\|_2
  = \| \theta_{i,j} \e_i + \M^{-1} \r \|_2 
  \le \theta_{\mmax} + \frac{ \| \r \|_2 }{ \sigma_{\mmin}(\M)},
\end{equation}
 and  $\| \r \|_2$ as 
\begin{align} \label{Exp: bound on tilde r}
 \|\r\|_2
 &= \|\r_{i,j} - \theta_{i,j} \r_{i, n_i}\|_2  \nonumber \\
 &\le (1+\theta_{\mmax}) \|\R\|_2 \nonumber \\
 &< \frac{1}{2} (1 - \theta_{\mmax}^2) \hat{\alpha}_{\mmin}. 
\end{align}
  The last inequality comes from the bound on $\|\R\|_2$ in this theorem.
Accordingly,
from inequalities (\ref{Exp: bound on the smallest singular value of S}), (\ref{Exp: bound on norm of c})
and (\ref{Exp: bound on tilde r}), we have $\|\c\|_2 < 1$.

\end{proof}

\subsection{Step 3} \label{Subsec: analysis of Step 3}
Let us move on to the analysis of Step 3.

\begin{theo} \label{Theo: cluster extraction}
 Let $\P = \U\Q + \R$ be of the form shown in (\ref{Exp: P}).
 Let $\Gamma = \{S_1, \ldots, S_k\}$ be the $k$-way partition of a graph corresponding to $\Q$.
 For the columns $\p_i$ of $\P$, take the normalized ones $\bar{\p}_i = \p_i / \|\p_i\|_2$.
 Assume that we have an element $u_i$ in $S_i$ for every $i = 1, \ldots, k$.
 Pick an element $v$ from $S_j$.
 If the $\ell$th column $\r_\ell$  of $\R$ satisfies 
 \begin{equation*}
  \|\r_\ell \|_2 < (17 - 12\sqrt{2}) \alpha_{\mmin}
 \end{equation*}
 for  $\ell = 1,\ldots,n$,
 then the chosen $i^* = \arg \max_{i=1, \ldots, k} \bar{\p}_{u_i}^\trans \bar{\p}_v$
 satisfies  $i^* = j$.
\end{theo}
Recall that $\alpha_{\mmin}$ is the smallest among all $\alpha_{i,j}$ for $i=1, \ldots, k, \ j=1, \ldots, n_i$.
Before going to the proof,  we derive the range of $\|\R\|_2$ that covers the ranges
imposed in  Theorems \ref{Theo: finding cluster representatives} and \ref{Theo: cluster extraction}.
From the definition,
we have $\alpha_{\mmin} = \min_{i=1,\ldots,k}\alpha_{i,1}$.
In addition, $\alpha_{i,1}$ can be written as
$\alpha_{i,1} = (\alpha_{i,1} / \alpha_{i,n_i}) \cdot \alpha_{i,n_i} = \theta_{i,1} \alpha_{i,n_i}$.
Hence, $\alpha_{\mmin} \ge \theta_{\mmin} \hat{\alpha}_{\mmin}$ holds.
Accordingly, if $\|\R\|_2 < (17-12\sqrt{2})\theta_{\mmin} \hat{\alpha}_{\mmin}$,
we see that $\|\r_\ell\|_2$ covers the range imposed in Theorem \ref{Theo: cluster extraction},
since
\begin{equation*}
 \|\r_\ell\|_2 \le \|\R\|_2 < (17-12\sqrt{2})\theta_{\mmin} \hat{\alpha}_{\mmin} \le (17-12\sqrt{2})\alpha_{\mmin}
\end{equation*}
for $\ell = 1,\ldots,n$.
Consequently, Theorems \ref{Theo: finding cluster representatives} and \ref{Theo: cluster extraction} imply that
we can identify a graph partition from the columns of $\P$, 
if
\begin{equation} \label{Exp: condition on R}
 \|\R\|_2 < \hat{\alpha}_{\mmin} \cdot \min \biggl\{ \frac{1}{2} (1-\theta_{\mmax}), \ (17-12\sqrt{2})\theta_{\mmin} \biggr\}.
\end{equation}

Now let us prove Theorem \ref{Theo: cluster extraction}.
First, we build Propositions \ref{Prop: ith element of bar z} and \ref{Prop: range of omega}.
In what follows, we use $\p_{i,j}$ to refer to the columns of $\P$.
Let 
\begin{equation*}
 \z_{i, j} := \U^\trans \p_{i,j} \in \Real^k
  \quad \mbox{and} \quad
  \bar{\z}_{i,j} := \frac{\z_{i,j}}{\|\z_{i,j}\|_2} \in \Real^k.
\end{equation*}
Since  $\p_{i,j} = \alpha_{i,j}\u_i + \r_{i,j}$,
we can express $\z_{i, j}$ as
\begin{equation*} 
\z_{i,j} = \alpha_{i,j} \e_{i} + \n_{i,j}
\end{equation*}
by letting $\n_{i,j} := \U^\trans \r_{i,j} \in \Real^k$.
Since $\U$ is orthogonal,
$\bar{\p}_{i,j}^\trans \bar{\p}_{u,v} = \bar{\z}_{i,j}^\trans \bar{\z}_{u,v}$ and 
$ \|\n_{i,j}\|_2 = \|\r_{i,j}\|_2 $.
We will examine the inner product of $\bar{\z}_{i,j}$ and $\bar{\z}_{u,v}$
instead of the one  of $\bar{\p}_{i,j}$ and $\bar{\p}_{u,v}$.
In the proposition below, we use $(\a)_i$ to denote the $i$th element of the vector $\a$.

\begin{prop}\label{Prop: ith element of bar z}
 Let $\bar{\z}_{i,j}$ be defined as above.
 The $i$th element is bounded from below:
\begin{equation*}
  (\bar{\z}_{i,j})_i \ge \frac{\alpha_{\mmin} - \|\n_{i,j}\|_2 }{\alpha_{\mmin} + \|\n_{i,j}\|_2 }.
 \end{equation*}
\end{prop}
\begin{proof}
 From the definition, 
\begin{equation*}
 \bar{\z}_{i,j}
  = \frac{\z_{i,j}}{\|\z_{i,j}\|_2}
  = \frac{\alpha_{i,j} \e_{i} + \n_{i,j}}{\|\alpha_{i,j} \e_{i} + \n_{i,j}\|_2}.
\end{equation*}
 Hence, the $i$th element of $\bar{\z}_{i,j}$ is given as 
\begin{equation*}
 (\bar{\z}_{i,j})_i = \frac{\alpha_{i,j} + (\n_{i,j})_i}{\|\alpha_{i,j} \e_{i} + \n_{i,j}\|_2}.
\end{equation*}
 Since $|(\n_{i,j})_i| \le \|\n_{i,j}\|_{\infty} \le \|\n_{i,j}\|_2$ holds, 
 we have, for the numerator, $\alpha_{i,j} + (\n_{i,j})_i \ge \alpha_{i,j} - \|\n_{i,j}\|_2$.
 Also, we have, for the square of the denominator, 
 \begin{align*}
  \|\alpha_{i,j} \e_{i} + \n_{i,j}\|_2^2
  &= \alpha_{i,j}^2 + 2 \alpha_{i,j} (\n_{i,j})_i + \|\n_{i,j}\|_2^2 \\
  &\le \alpha_{i,j}^2 + 2 \alpha_{i,j} \|\n_{i,j}\|_2 + \|\n_{i,j}\|_2^2 \\
  &= (\alpha_{i,j} + \|\n_{i,j}\|_2)^2.
 \end{align*}
 This means $\|\alpha_{i,j} \e_{i} + \n_{i,j}\|_2 \le \alpha_{i,j} + \|\n_{i,j}\|_2$.
Accordingly, the $i$th element of $\bar{\z}_{i,j}$ is bounded from below as follows.
\begin{equation*}
 (\bar{\z}_{i,j})_i \ge \frac{\alpha_{i,j} - \|\n_{i,j}\|_2}{\alpha_{i,j} + \|\n_{i,j}\|_2}.
\end{equation*}
For some nonnegative real number $c$,
the function $f(x) = \frac{x - c}{x + c}$ for
positive real numbers $x$ is monotonically nondecreasing.
Consequently, since $\alpha_{i,j} \ge \alpha_{\mmin}$, we obtain
\begin{equation*}
 (\bar{\z}_{i,j})_i 
  \ge \frac{\alpha_{i,j} - \|\n_{i,j}\|_2}{\alpha_{i,j} + \|\n_{i,j}\|_2}
  \ge \frac{\alpha_{\mmin} - \|\n_{i,j}\|_2}{\alpha_{\mmin} + \|\n_{i,j}\|_2}.
\end{equation*}
\end{proof}

For some real number $\xi$ satisfying $0 \le \xi \le 1$, let 
\begin{equation*}
 C(i, \xi) := \{\x \in \Real^k : \|\x\|_2 = 1, \ x_i \ge \xi  \}
\end{equation*}
where $x_i$ is the $i$th element of $\x$.
If $\|\n_{i,j} \|_2 \le \alpha_{\mmin}$, Proposition \ref{Prop: ith element of bar z} tells us that 
\begin{equation*} 
 \bar{\z}_{i,j} \in C \biggl( i, \ \frac{\alpha_{\mmin} - \|\n_{i,j}\|_2 }{\alpha_{\mmin} + \|\n_{i,j}\|_2 } \biggl).
\end{equation*}
Obviously, the inner product of two elements from the same set $C(i, \xi)$ is large,
while that of two elements from different sets $C(i, \xi)$ and $C(j, \xi)$ with $i \neq j$ is small.
In the lemma below, we present bounds on those inner products.

\begin{lemm} \label{Lemm: bound on a'b}
\begin{enumerate}[{\normalfont (a)}]
 \item Let $\a, \b \in C(i, \xi)$. Then, $\a^\trans \b \ge  2 \xi^2 - 1$.
 \item Let $\a \in C(i, \xi)$ and $\b \in C(j, \xi)$ for $i \neq j$.
       Then, $\a^\trans \b \le - \xi^2 + 2\sqrt{1 - \xi^2} + 1$.
\end{enumerate}
\end{lemm}
It is easy to prove this lemma. We thus put the proof in the Appendix.
The proposition below immediately follows from the lemma.
\begin{prop} \label{Prop: range of omega}
 Assume that we have an element $\a_i$ in $C(i, \xi)$ for every $i = 1,\ldots,k$.
 Pick an element $\a$ from $C(j, \xi)$.
 If $\xi > \frac{2\sqrt{2}}{3}$, then
 the chosen $i^* = \arg \max_{i = 1,\ldots,k} \a_i^\trans \a$ satisfies $i^* = j$.
 \end{prop}
\begin{proof}
 Lemma \ref{Lemm: bound on a'b} tells us that
 $\a_i^\trans \a \ge 2 \xi^2 - 1$ if $i = j$;
 otherwise, $\a_i^\trans \a \le  - \xi^2 + 2\sqrt{1 - \xi^2} + 1$.
 Let us examine the range of $\xi$ such that
 $2\xi^2 - 1$ is strictly larger than $- \xi^2 + 2\sqrt{1 - \xi^2} + 1$.
 Let $f$ be a function defined by 
 $f(x) = 3x^2 -2 \sqrt{1 - x^2} -2$ for $0 \le x \le 1$.
 This function is monotonically increasing for $ 0 < x < 1$;
 $f(0) = -4$ and $f(1) = 1$; and 
 $f(x) = 0$ when $x = \frac{2\sqrt{2}}{3}$.
 We thus see that $f$ takes a positive value if $x > \frac{2\sqrt{2}}{3}$.
 Accordingly, if $\xi > \frac{2\sqrt{2}}{3}$,
 \begin{equation*}
  \a_j^\trans \a \ge 2 \xi^2 - 1 > - \xi^2 + 2\sqrt{1 - \xi^2} + 1 \ge \a_i^\trans \a
 \end{equation*}
 holds for every $i \in \{1, \ldots, k\} \setminus \{j\}$.
 Consequently,  the choice of $i^* = \arg \max_{i=1, \ldots, k} \a_i^\trans \a$ satisfies $i^* = j$.
\end{proof}

We are now ready to prove Theorem \ref{Theo: cluster extraction}.
\begin{proof}[Proof of Theorem \ref{Theo: cluster extraction}]
 Proposition \ref{Prop: ith element of bar z} implies 
 \begin{equation*}
  \bar{\z}_{i,j} \in C\biggl(i, \ \frac{\alpha_{\mmin} - \|\n_{i,j}\|_2}{\alpha_{\mmin} + \|\n_{i,j}\|_2} \biggr)
 \end{equation*}
 if $\|\n_{i,j}\|_2 \le \alpha_{\mmin}$.
 This theorem imposes the condition that $\|\r_{i,j}\|_2 < (17-12\sqrt{2})\alpha_{\mmin}$.
 Hence, $\bar{\z}_{i,j}$ belongs to $C(i, \xi)$  such that $\xi$ satisfies $\frac{2\sqrt{2}}{3} < \xi \le 1$, since  
 \begin{equation*}
  \frac{\alpha_{\mmin} - \|\n_{i,j}\|_2}{\alpha_{\mmin} + \|\n_{i,j}\|_2}
  = \frac{\alpha_{\mmin} - \|\r_{i,j}\|_2}{\alpha_{\mmin} + \|\r_{i,j}\|_2} > \frac{2\sqrt{2}}{3}.
\end{equation*}
 The equality comes from that $\n_{i,j} = \U^\trans \r_{i,j}$ and $\U$ is orthogonal.
 We have $\bar{\p}_{u_i}^\trans \bar{\p}_v = \bar{\z}_{u_i}^\trans \bar{\z}_v$.
 In addition, this theorem assumes that we have $u_i \in S_i$ for $i=1,\ldots,k$ and picks $v \in S_j$.
 Hence, $\bar{\z}_{u_i} \in C(i, \xi)$ for $i = 1, \ldots, k$ and $\bar{\z}_{v} \in C(j, \xi)$. 
 Accordingly, Proposition \ref{Prop: range of omega} ensures that 
 the chosen $i^* = \arg \max_{i=1,\ldots,k} \bar{\p}_{u_i}^\trans \bar{\p}_v = \bar{\z}_{u_i}^\trans \bar{\z}_v$
 satisfies  $i^* = j$.
\end{proof}

\subsection{Proofs of Theorem \ref{Theo: main result} and Corollary 
  \ref{Coro: obtained from main result}} \label{Subsec: proof of our main results}  
Theorem \ref{Theo: main result} is proved  using Theorems \ref{Theo: structure theorem},
\ref{Theo: finding cluster representatives} and \ref{Theo: cluster extraction}.
\begin{proof}[Proof of Theorem \ref{Theo: main result}]
 Let a $k$-way partition $\Gamma = \{S_1, \ldots, S_k\}$ of a graph $G$ be optimal.
 We take the normalized indicators $\bar{\g}_1, \ldots, \bar{\g}_k$ of $S_1, \ldots, S_k$.
 Let $\Q = [\bar{\g}_1, \ldots, \bar{\g}_k]^\trans \in \Real^{k \times n}$.
 Step 1 of ELLI computes the bottom $k$ eigenvectors
 $\f_1, \ldots, \f_k$ of the normalized Laplacian $\LC$ of $G$.
 Let  $\P = [\f_1, \ldots, \f_k]^\trans \in \Real^{k \times n}$.
 A graph $G$ satisfies $\Upsilon  > 4k / (\theta\alpha)^2$.
 This implies that the relation $\Upsilon \ge k$ holds,
 since  we have $0 \le \theta \alpha \le 1$ from the definitions of $\theta$ and $\alpha$.
 Hence, Theorem \ref{Theo: structure theorem} ensures that 
 there is some orthogonal matrix $\U \in \Real^{k \times k}$ such that
 $\P = \U\Q + \R$ and 
 \begin{equation*}
  \|\R\|_2 \le  2\sqrt{\frac{k}{\Upsilon}} = 2 \sqrt{\frac{k \cdot \phi_k(G)}{\lambda_{k+1}}}.  
 \end{equation*}
 Since $G $ satisfies $\Upsilon  > 4k / (\theta\alpha)^2$, we have 
 \begin{equation*}
  \Upsilon = \frac{\lambda_{k+1}}{\phi_k(G)} > \frac{4k}{(\theta\alpha)^2}
   \equivSym
   \frac{4k \cdot \phi_k(G)}{\lambda_{k+1}} < (\theta\alpha)^2.
 \end{equation*}
 This implies $\|\R\|_2 < \theta \alpha$.
 Recalling the range  shown in (\ref{Exp: condition on R})
 in Section \ref{Subsec: analysis of Step 3},
 we see that $\|\R\|_2$ lies within both ranges 
 imposed in Theorems \ref{Theo: finding cluster representatives} and \ref{Theo: cluster extraction}.
 Theorem \ref{Theo: finding cluster representatives}  ensures that the set $I$ constructed in Step 2
 is the representative node set of the optimal $k$-way partition $\Gamma$.
 Let $u_1, \ldots, u_k$ be elements of $I$ such that $u_i \in S_i$ for $i=1,\ldots,k$.
 Theorem \ref{Theo: cluster extraction} ensures that, if $v \in S_i$, then Step 3 adds $v$ to the set $T_i$.
 Consequently, the obtained set $T_i$ coincides with the cluster $S_i$
 in the optimal $k$-way partition $\Gamma$ for $i=1,\ldots,k$.
\end{proof}

Corollary \ref{Coro: obtained from main result} immediately follows from Theorem \ref{Theo: main result}.

\begin{proof}[Proof of Corollary \ref{Coro: obtained from main result}]
 The first and second conditions of Corollary \ref{Coro: obtained from main result} lead to 
\begin{align*}
 \theta_{\mmax} \le 1 - \frac{1}{\sqrt{k^p}}
 \quad \mbox{and} \quad
 \theta_{\min} \ge \frac{1}{\sqrt{k^q}}.
\end{align*}
 From the inequalities, we obtain 
 \begin{align*}
 \frac{1}{2} (1 - \theta_{\mmax}) \ge \frac{1}{2\sqrt{k^p}}  \ge \frac{1}{2\sqrt{k^p + k^q}},
 \end{align*}
 and
 \begin{align*}
  (17 - 12 \sqrt{2}) \theta_{\mmin} \ge \frac{1}{50 \sqrt{k^q}} \ge \frac{1}{50 \sqrt{k^p + k^q}}.
 \end{align*}
 Hence, we can bound $\theta$ as follows.
 \begin{align*} 
  \theta = \min \left\{ \frac{1}{2}(1 - \theta_{\mmax}), (17 - 12\sqrt{2}) \theta_{\mmin} \right\} 
  \ge \frac{1}{50 \sqrt{k^p + k^q}}. 
  \end{align*}
 Also, using the third condition of Corollary \ref{Coro: obtained from main result}, we can bound $\alpha$ as follows.
\begin{align*} 
 \alpha \ge \frac{1}{\sqrt{k^r}}.
\end{align*}
 The bounds on $\theta$ and $\alpha$ imply
 \begin{align*}
 \frac{4k}{(\theta \alpha)^2} \le 10^4 \cdot k^{r+1}(k^p + k^q) \le 2 \cdot 10^4 \cdot k^{(r+1) \cdot { \max \{p, q\}}}.
 \end{align*}
 Accordingly, if $\Upsilon = \Omega \left( k^{(r+1) \cdot { \max \{p, q\}}} \right) $, 
 Theorem \ref{Theo: main result} ensures that the output of ELLI is the optimal $k$-way partition of a graph.
\end{proof}

\section{Connection with Computing Separable NMFs} \label{Sec: connection with separable NMFs}

Finding the representative node set of a graph partition is connected with computing separable NMFs.
In what follows, we will use the symbol $\Real_+^{d \times n}$
to denote the set of \by{d}{n} nonnegative matrices.
Let a nonnegative matrix $\A \in \Real_+^{d \times n}$ have a factorization
such that $\A = \S\C$ for nonnegative matrices
$\S \in \Real_+^{d \times k}$ and $\C \in \Real_+^{k \times n}$.
This factorization is referred to as NMF.
The NMF of $\A$ is called \emph{separable} if 
it can be further factorized into
\begin{equation} \label{Exp: separable matrix} 
 \A = \S \C \ \mbox{for} \ \S \in \Real_+^{d \times k} \ \mbox{and} \ \C = [\I, \H]\Pib \in \Real_+^{k \times n}.
\end{equation}
Here, $\I$ is a \by{k}{k} identity matrix,
$\H$ is a \by{k}{(n-k)} nonnegative matrix, and $\Pib$ is an \by{n}{n} permutation matrix.
A separable NMF is a special case of an NMF that 
satisfies the further condition that 
all columns of $\S$ appear in those of $\A$.
Given $\A$ of the form shown in (\ref{Exp: separable matrix}),
a separable NMF problem is one of finding a column index set $I$ with $k$ elements such that $\A(I) = \S$.
Here, $\A(I)$ denotes the submatrix of $\A$ indexed by $I$, i.e., $\A(I) = [\a_i : i \in I]$ 
for the $i$th column $\a_i$ of $\A$.
We call such an index set $I$ the \emph{basis index set}.
The notion of separability was introduced by Donoho and Stodden \cite{Don03}.
Arora et al.\ \cite{Aro12a} showed that separable NMF problems are solvable in polynomial time.

Let $\A$ be of the form shown in (\ref{Exp: separable matrix}).
The perturbed matrix is given by $\tilde{\A} = \A + \R$ for $\R \in \Real^{k \times n}$,
which is the noise added to $\A$.
Let us choose one algorithm for solving separable NMF problems.
Given $\tilde{\A}$ and $k$, we say that the algorithm is \emph{robust to noise} 
if it finds a column index set $I$ with $k$ elements such that $\tilde{\A}(I)$ is close to $\S$.
There are several algorithms that have been shown to be robust to noise.
SPA is an algorithm for solving separable NMF problems, and Gillis and Vavasis \cite{Gil14a} examined its robustness.
They built the following setup for their analysis.
A matrix $\A \in \Real^{d \times n}$ is factorized into
\begin{equation} \label{Exp: setup by Gillis and Vavasis}
 \A = \S \C \ \mbox{for} \ \S \in \Real^{d \times k} \ \mbox{and} \ \C = [\I, \H]\Pib \in \Real_+^{k \times n}
\end{equation}
satisfying the following two conditions.
\begin{enumerate}[({A}1)]
 \item $\S$ is full column rank.
 \item Every column $\h_i$ of $\H$ satisfies $\|\h_i\|_1 \le 1$.
\end{enumerate}
In their setup, $\C$ has to be nonnegative, but $\S$ does not necessarily have to be nonnegative, 
unlike $\S$ in (\ref{Exp: separable matrix}).
Let $\A$ be of the form shown in (\ref{Exp: setup by Gillis and Vavasis}) and 
assume that it satisfies conditions (A1) and (A2).
The perturbed matrix is given by $\tilde{\A} = \A + \R$ for $\R \in \Real^{k \times n}$.
Given $\tilde{\A}$ and $k$ as input,
Gillis and Vavasis showed in Theorem 3 of \cite{Gil14a} that, 
if $\| \R \|_2$ is smaller than some threshold, 
then SPA finds a column index set $I$ such that $\tilde{\A}(I)$ is close to $\S$.
Gillis \cite{Gil14b} developed a successive nonnegative projection algorithm (SNPA),
and showed in Theorem 3.22 of the paper that a similar robustness result holds for SNPA 
even if condition (A1) is replaced with a weaker condition.
ELLI uses ER.
The robustness of ER was shown in Theorem 9 of \cite{Miz14}, 
which assumes that conditions (A1) and (A2) are satisfied, and, in addition, $\S$ is nonnegative.

Let us go back to $\P = \U\Q + \R$ shown in (\ref{Exp: P}).
By choosing an \by{n}{n} permutation matrix $\Pib$,
we can rearrange the columns of $\Q$ such that 
\begin{equation*}
 \Q = 
  \left[
  \begin{array}{ccc | ccc | c | ccc}
   \alpha_{1,n_1} &         &                  & \alpha_{1,1} & \cdots & \alpha_{1, n_1-1} &        &              &        & \\
                  &  \ddots &                  &              &        &                   & \ddots &              &        & \\
                  &         &  \alpha_{k,n_k}  &              &        &                   &        & \alpha_{k,1} & \cdots & \alpha_{k, n_k-1}
  \end{array}
  \right] \Pib \in \Real^{k \times n}
\end{equation*}
where $\alpha_{i,j} = \sqrt{d_{i,j} / \mu(S_i)}$.
Let $\V$ denote the \by{k}{k} submatrix consisting of the first $k$ columns,
$\V = \mbox{diag}(\alpha_{1,n_1}, \ldots, \alpha_{k,n_k})$.
This is nonsingular since all $d_{i,j}$ are positive.
Let $\H$ denote a matrix obtained by the product of $\V^{-1}$ and
the remaining \by{k}{(n-k)} submatrix,
\begin{equation*}
 \H =
  \left[
   \begin{array}{ccc|ccc|c|ccc}
     \theta_{1,1} & \cdots  & \theta_{1, n_1-1}
    &  &  & 
    & 
    &  &  &   \\

    &  &  &
     \theta_{2,1} & \cdots  & \theta_{2, n_2-1}    
    & 
    &  &  & \\
      &  &
    &  &  &
    & \ddots 
    &  &  & \\
      &  &
    &  &  &
    & 
    & \theta_{k,1}  & \cdots  & \theta_{k, n_k-1} \\
   \end{array}
  \right] \in \Real^{k \times (n-k)}
\end{equation*}
where $\theta_{i,j} = \alpha_{i,j} / \alpha_{i,n_i} = \sqrt{d_{i,j} / d_{i,n_i}}$.
Accordingly, $\Q$ can be rewritten as 
\begin{equation*}
 \Q = \V \C  \  \mbox{for} \ \V \in \Real_+^{k \times k} \ \mbox{and} \
 \C = [\I,\H]\Pib \in \Real_+^{k \times n}.
\end{equation*}
The above shows that $\Q = \V\C$ is NMF and separable and 
the basis index set corresponds to the representative node set of the $k$-way partition  of a graph.
We can therefore see that finding the representative node set of a graph partition is equivalent to 
finding the basis index set of $\Q$.
The grouping stage of spectral clustering needs to deal with $\P$ rather than $\Q$.
This matrix is given as 
\begin{equation*}
 \P = \U\Q + \R = \U\V\C + \R.
\end{equation*}
The matrix $\U\V$ is not necessarily nonnegative.
Nevertheless,  we can perform SPA and SNPA on $\P$ and ensure their robustness if $\|\R\|_2$ is small,
because the matrix $\U\V\C$ satisfies conditions (A1) and (A2).
Indeed, $\U\V$ is \by{k}{k} and nonsingular because $\U$ and $\V$ both are.
In addition, $\C = [\I, \H]\Pib$ is nonnegative and
every column $\h_i$ of $\H$ satisfies $\|\h_i\|_1  \le \theta_{\mmax} \le 1$.
Meanwhile, Theorem 9 of \cite{Miz14}, which describes the robustness of ER,
is invalid for $\P$, because $\U\V$ is required to be nonnegative.
\begin{remark}
In Section \ref{Subsec: analysis of Step 2}, we mentioned the size of $\|\R\|_2$.
That is, if $\Upsilon \ge k$ holds, 
then we have $\|\R\|_2 \le 2\sqrt{k / \Upsilon}$ for $\Q$ corresponding to the optimal $k$-way partition of a graph.
Hence, if the graph is well-clustered, $\|\R\|_2$ is small.
\end{remark}

\section{Experiments} \label{Sec: experiments}
We describe experiments conducted on synthetic and real data.
The details of the implementations of ELLI and KSC are as follows.

\begin{itemize}
 \item (ELLI)~Step 1 computes the bottom $k$ eigenvectors of the normalized Laplacian.
       For this computation, we used the MATLAB command \verb|eigs|, 
       choosing the value \verb|'sa'| in the input argument.
       Step 2 computes an origin-centered MVEE for the set of points.
       We used an interior-point method working within a cutting plane framework.
       Our implementation followed Algorithm 3 of \cite{Miz14}
       and used the interior-point method in the SDPT3 software package \cite{Toh99b}.

 \item (KSC)~Our implementation followed the algorithm in \cite{Lux07}
       that is referred to as normalized spectral clustering according to Shi and Malik.
       In the embedding stage, we used the MATLAB command \verb|eigs| with the same settings
       in ELLI for the eigenvector computation.
       We then constructed the spectral embedding map $F$ of the form shown in (\ref{Exp: spectral embedding map}) 
       with a scaling factor $s_u = 1 / \sqrt{d_u}$ for the degree $d_u$ of node $u$.
       In the grouping stage,
       we adopted the $k$-means++ algorithm \cite{Art07} for finding a $k$-way partition of a set $X$ of points $F(1), \ldots, F(n)$,
       since our preliminary experiments indicated that the $k$-means++ algorithm outperformed Lloyd's algorithm \cite{Llo82}.
       To perform it, we used the MATLAB command \verb|kmeans|, 
       choosing the following values in the input arguments:
       \verb|'Start'|, \verb|'plus'|, \verb|'EmptyAction'|, \verb|'singleton'|, \verb|'MaxIter'|, \verb|1000|.
       The value in the argument \verb|'MaxIter'| specifies the maximum number of iterations.
       We set it to 1000. 
	    
\end{itemize}
The experiments were conducted on 
an Intel Xeon Processor E5-1620 with 32 GB memory running MATLAB R2016a.

\subsection{Synthetic Data} \label{Subsec: synthetic data}
The first experiments assessed how close the conductance of the $k$ clusters found
by ELLI and KSC were to the $k$-way conductance $\phi_k(G)$ of the graph.
As it is hard to compute $\phi_k(G)$,
we used synthetic data  for which an upper bound on $\phi_k(G)$ is easily obtainable.
Specifically, we synthesized adjacency matrices and 
constructed the normalized Laplacians from them.

We will use the following notation to describe the generation procedure.
For integers $p$ and $q$ with $p \le q$,
the notation $[p:q]$ indicates the set of all integers
from $p$ to $q$. For instance, $[1:3] = \{1,2,3\}$.
For a matrix $\A \in \Real^{m \times n}$ with $m \ge n$,
we take a set $K = [p:q]$ satisfying $K \subset [1:n]$.
The symbol $\A_K$ denotes the submatrix $[a_{ij} : i \in K, \ j \in K]$  of $\A$,
where $a_{ij}$ is the $(i,j)$th element of $\A$,
\begin{equation*}
 \A_K = \left[
	 \begin{array}{ccc}
	  a_{pp}  & \cdots  & a_{pq} \\
	  \vdots  & \ddots  & \vdots \\
	  a_{qp}  & \cdots  & a_{qq} \\
	 \end{array}
	\right].
\end{equation*}
The following procedure was used to make the adjacency matrix.
\begin{enumerate}
 \item Choose an \by{n}{n} symmetric matrix $\M$ such that
       the diagonal elements are all $0$ and the other elements 
       lie in the interval $\{x : 0 < x < 1\}$.

 \item Choose $k$ integers $n_1, \ldots, n_k$ satisfying $n = n_1 + \cdots + n_k$ and
       construct 
       \begin{equation*}
	S_i = \Biggl[ \sum_{\ell=1}^{i-1} n_\ell + 1 :  \sum_{\ell=1}^{i} n_\ell \Biggr]
       \end{equation*}
       for $i = 1, \ldots, k$.

 \item Let $\B$ be an \by{n}{n} block diagonal matrix
       \begin{equation*}
	 \left[
	  \begin{array}{ccc}
	   \M_{S_1} &        &  \\
	   & \ddots &           \\
	   &        & \M_{S_k} \\
	  \end{array}
		     \right]
       \end{equation*}
       where $\M_{S_1}, \ldots, \M_{S_k}$ are the submatrices of $\M$ indexed by $S_1, \ldots, S_k$.
       Also, let $\R$ be an off-block diagonal matrix $\frac{1}{2} (\M - \B)$.
       Choose a value of the intensity parameter $\delta$ from $0$ to $2$
       and generate an \by{n}{n} symmetric matrix 
       $\W = \B + \delta \R$.
\end{enumerate}

The generated matrix $\W$ is regarded as the adjacency matrix for some graph $G$,
and the constructed sets $S_1, \ldots, S_k$ are clusters in the $k$-way partition of $G$.
When $\delta = 0$,
the matrix $\W$ is clean block diagonal,
and the corresponding graph $G$ consists of $k$ connected components.
As $\delta$ increases, the block structure gradually disappears.
When $\delta = 2$, the original matrix $\M$ is reacquired.
A simple calculation shows that the conductance $\phi(S_i)$ of $S_i$ is
\begin{equation*}
 \phi(S_i) = \frac{\delta}{c_i + \delta}.
 \end{equation*}
Here, $c_i$ is a positive number determined by $\M$.
Hence, $\delta / (c + \delta)$ with $c = \min \{c_1, \ldots, c_k\}$ is 
the maximum of $\phi(S_1), \ldots, \phi(S_k)$.
Hence, we can see that the $k$-way conductance $\phi_k(G)$ of $G$ 
is bounded from above by the function $f(x) = x / (c + x)$.
It may serve as a good upper bound on $\phi_k(G)$.
In particular, the bound can be tight if $\delta$ is sufficiently small.

On the basis of the above procedure,
the experiments generated two types of dataset: balanced and unbalanced.
We set $n = 10,000$ and constructed $\Gamma_1$ and $\Gamma_2$ as follows.
\begin{itemize}
 \item $\Gamma_1 = \{S_1, \ldots, S_{50}\}$ with 
	    \begin{equation*}
	      |S_1| = \cdots = |S_{50}| = 200.
	    \end{equation*}
 \item $\Gamma_2 = \{S_1, \ldots, S_{143}\}$ with 
 \begin{equation*}
  |S_1| = |S_2| = |S_3| = 1,000 \quad \mbox{and} \quad |S_4| = \cdots = |S_{143}| = 50.
 \end{equation*}
\end{itemize}
We constructed adjacency matrices with
an intensity parameter $\delta$ running from $0$ to $2$ in increments of $0.1$
for each $\Gamma_1$ and $\Gamma_2$.
The set of adjacency matrices for $\Gamma_1$ was the balanced dataset,
while that of $\Gamma_2$ was the unbalanced dataset.

The experiments ran ELLI and KSC on the normalized Laplacians produced from the datasets.
The quality of the obtained clusters was evaluated by the maximum value
of cluster conductance (MCC), defined by
\begin{equation*}
 \max\{\phi(S_1), \ldots, \phi(S_k)\}
\end{equation*}
for the output $\{ S_1, \ldots, S_k \}$ of the algorithm.
KSC repeated the $k$-means++ algorithm 100 times for each input.
Hence, the evaluation of clusters returned by KSC was
the average MCC over 100 trials.

Figure \ref{Fig: result of the first experimentsa} shows the experimental results.
The top two figures are the results of ELLI and KSC on the balanced dataset,
and the bottom two figures are those of the unbalanced dataset.
The red points in the left figures are the MCC of ELLI,
while those in the right figures are the average MCC of KSC.
The black dotted line depicts the function $f(x) = x / (c + x)$
that serves as an upper bound on graph conductance.
We can see from the figures that the MCC of ELLI approaches the upper bound, 
and it seems to be lower than the average MCC of KSC on both datasets.
Figure \ref{Fig: difference of MCC by ELLI and KSC} clarifies the differences between them
for the balanced dataset and unbalanced dataset.
Each red point plots the average MCC of KSC minus the MCC of ELLI.
We clearly see from the figures that the MCC of ELLI is consistently
below the average MCC of KSC except for $\delta = 0$.

\begin{figure}[p]
 \centering
 \includegraphics[width=\linewidth]{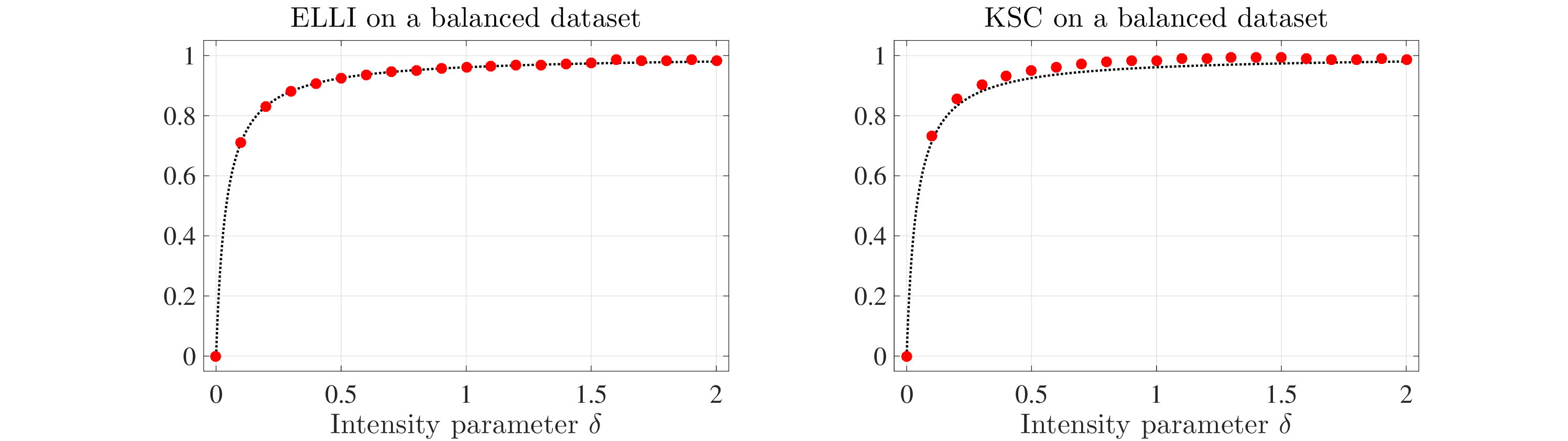}

 \bigskip 
 \includegraphics[width=\linewidth]{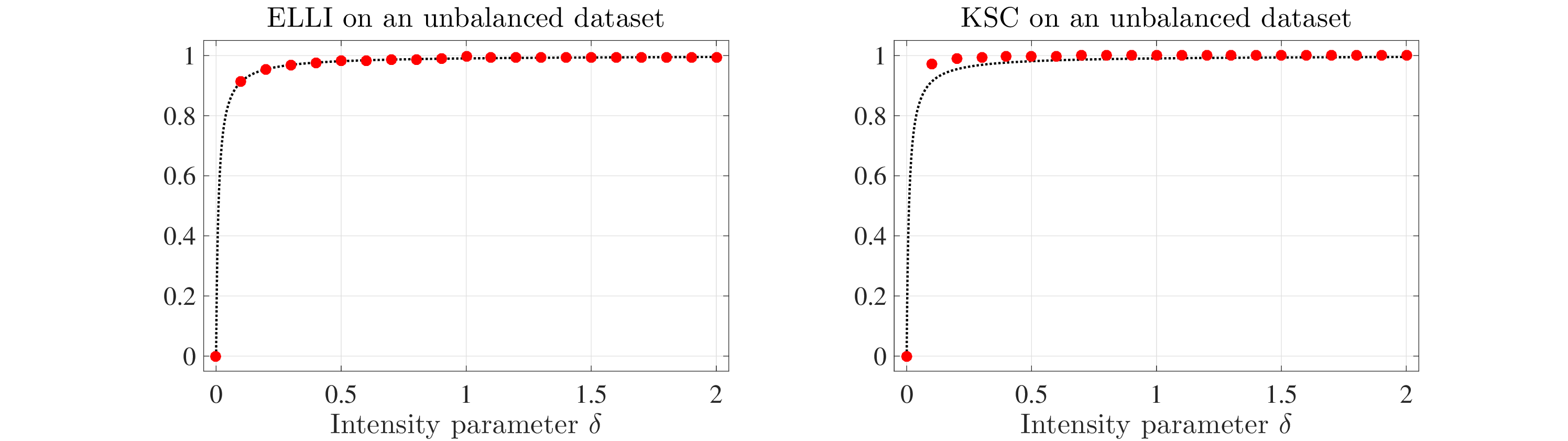}
 \caption{Results of ELLI and KSC
 for balanced and unbalanced datasets.
 The red points in the left figures are the MCCs of ELLI for
 each adjacency matrix  with $\delta$,
 while those in the right figures are the average MCCs over the 100 trials of
 the $k$-means++ algorithm for each adjacency matrix  with $\delta$.
 The black dotted line is 
 $f(x) = x / (c+x)$, an upper bound on graph conductance.}
 \label{Fig: result of the first experimentsa}
 
 \bigskip  \bigskip  \bigskip
 \includegraphics[width=\linewidth]{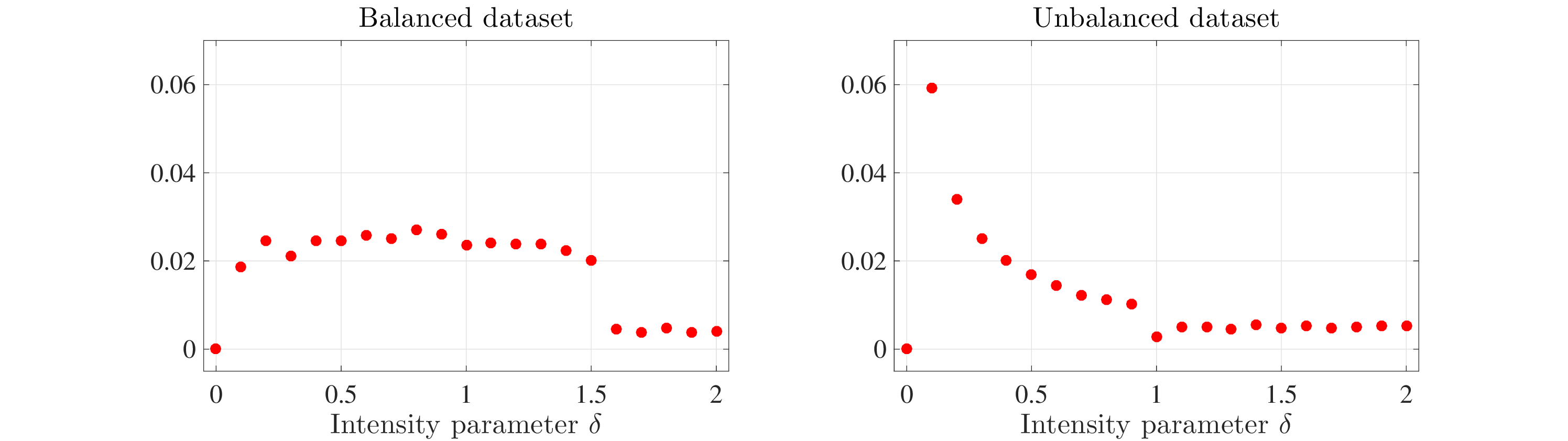}
 \caption{Difference between the MCC of ELLI and the average MCC of KSC:
 balanced dataset (left) and unbalanced dataset (right).
 Each red point is the value had by subtracting the MCC of ELLI from the average MCC of KSC.
 Thus, a red point lying on the positive side indicates
 that the MCC of ELLI is below the average MCC of KSC.}
 \label{Fig: difference of MCC by ELLI and KSC}
\end{figure}

\subsection{Real Data} \label{Subsec: real data}
The second experiments assessed how effective ELLI is at clustering real data.
We chose several image databases
containing images that had been categorized into classes by human judges.
Then, we constructed image datasets by using the whole or some parts of the databases.
We evaluated how well the clusters found by ELLI matched
the classes of the datasets.

For comparison, we tested the graph regularized NMF (GNMF) 
proposed in \cite{Cai11} and KSC.
GNMF is known to be effective at clustering.
Before describing the details of the experiments,
let us briefly describe the clustering algorithm with the use of GNMF.
Let $n$ data vectors $\a_1, \ldots, \a_n \in \Real^d$ be nonnegative.
The clustering algorithm maps the $n$ data vectors to $n$ points
in a lower dimensional space;
then it applies the $k$-means method to the points.
The mapping is constructed using GNMF.
For a nonnegative matrix $\A \in \Real^{d \times n}_+$ that stacks $\a_1, \ldots, \a_n$ in columns,
the GNMF problem is one of finding two nonnegative matrices $\X \in \Real^{d \times k}_+$
and $\Y \in \Real^{k \times n}_+$ that minimize the cost function,
\begin{equation*}
 f(\X,\Y) = \|\A - \X\Y \|_F^2 + \lambda \cdot \trace(\Y \L \Y^\trans).
\end{equation*}
Here, $\lambda$ is a positive parameter the user specifies, and
$\L$ is the Laplacian of the adjacency matrix $\W$ formed from the data vectors.
The symbol $\trace(\cdot)$ denotes the trace of a matrix.
This is an NMF problem with a regularization term $\trace(\Y \L \Y^\trans)$.
After solving it heuristically,
the clustering algorithm regards the columns $\y_1, \ldots, \y_n$ 
of $\Y$ as the representations of the data vectors $\a_1, \ldots, \a_n$ in $\Real^k$
and applies the $k$-means method to $\y_1, \ldots, \y_n$.
If two data vectors $\a_i$ and $\a_j$ are close,
so should be the corresponding two columns $\y_i$ and $\y_j$.
The regularization term serves to make $\y_i$ and $\y_j$ even closer.
Indeed, we can rewrite the term as 
\begin{equation*}
 \trace(\Y \L \Y^\trans) = \frac{1}{2} \sum_{i=1}^{n} \sum_{j=1}^{n} w_{ij} \|\y_i - \y_j \|_2^2
\end{equation*}
for the $(i,j)$th element $w_{ij}$ of the adjacency matrix $\W$.
If $\a_i$ and $\a_j$ are close together, then, $w_{ij}$ takes a high value.
Hence, we expect that the columns $\y_i$ and $\y_j$ of $\Y$
found by solving the GNMF problem are also close together.
The code for GNMF is available from the website of the first author in \cite{Cai11}.
In the experiment, we used it for computing $\Y$.
For the clustering of the columns of $\Y$,
we used the MATLAB code \verb|kmeans| with the same settings as KSC.

The experiment used five image databases: 
EMNIST \cite{Coh17},  ETL, Fashion-MNIST \cite{Xia17}, MNIST \cite{Lec98}, and NDL.
The MNIST database is a standard benchmark for evaluating clustering algorithms.
It contains the images of ten handwritten digits from $0$ to $9$.
We used all images in it.
MNIST is derived from the NIST Special Database.
The EMNIST database is an extension of MNIST
that consists of six datasets.
Among them, the EMNIST-Balanced dataset
contains the images of handwritten alphabet letters and digits.
We used this dataset.
The Fashion-MNIST database contains images of fashion products
from ten categories, such as T-shirts, trousers and pullovers.
We used all images in it.
ETL and NDL are image collections of Japanese characters.
The ETL dataset consists of all images of katakana characters
in the ETL1 dataset of the ETL Character Database,
an image collection of handwritten and machine-printed letters and digits
collected by AIST, Japan.
The NDL dataset consists of images of hiragana characters
from the image databases at the website of the 
National Diet Library of Japan\footnote{\url{http://lab.ndl.go.jp/cms/}}.
The character images were extracted from documentary materials published
from 1900 to 1977, which are available in the  National Diet Library Digital Collections.

Except for NDL,
all of the images in the datasets were grayscale.
Some of the NDL images were RGB;
we transformed them into grayscale images 
by using the MATLAB command \verb|rgb2gray|.
The sizes of the images in each dataset were equal.
The $n$ grayscale images in a dataset were represented
as vectors $\a_1, \ldots, \a_n$.
Here, given an image size of $h \times w$ pixels,
an image vector $\a_i$ is $(h \times w)$-dimensional and 
the value of each element is a grayscale intensity at the corresponding pixel.
Table \ref{Tab: image datasets} summarizes the dimension $d$ of the image vectors,
the number $n$ of images, and the number $k$ of classes in the dataset.

\begin{table}[h]
 \caption{Image datasets used in the experiments.} \label{Tab: image datasets}
 \centering
  \begin{tabular}{llrrrr} 
   \toprule
                   &  Description of images       & \multicolumn{2}{c}{Dimension $d$} & \# Data $n$  & \# Classes $k$ \\
   \midrule   
   EMNIST-Balanced &  Alphabet letters and digits & $784$    & $(28 \times 28)$ & $131,600$ & $47$ \\
   ETL             &  Katakana characters            & $4,032$  & $(64 \times 63)$ & $71,959$  & $51$ \\
   Fashion-MNIST   &  Fashion products            & $784$    & $(28 \times 28)$ & $70,000$  & $10$ \\
   MNIST           &  Digits                      & $784$    & $(28 \times 28)$ & $70,000$  & $10$ \\
   NDL             &  Hiragana characters            & $2,304$  & $(48 \times 48)$ & $80,000$  & $73$ \\
   \bottomrule
  \end{tabular}
  
\end{table}

Adjacency matrices were formed from the image vectors.
The construction was based on the procedure suggested in Section 2.2 of \cite{Lux07}.
Let $\a_1, \ldots, \a_n \in \Real^d$ be image vectors in a dataset,
and assume that $\a_1, \ldots, \a_n$ are nonnegative.
This assumption is a natural one, as 
the value of each element represents a grayscale intensity.
The similarity between $\a_i$ and $\a_j$ is evaluated using
\begin{equation*}
 s(\a_i, \a_j) = \frac{\a_i^\trans \a_j}{\|\a_i\|_2 \|\a_j\|_2}.
\end{equation*}
The value ranges from $0$ to $1$.
It is close to $1$ if $\a_i$ is nearly parallel to $\a_j$,
while it is close to $0$ if $\a_i$ and $\a_j$ are well spread.
The EMNIST-Balanced dataset contains over one-hundred-thousand images.
If we computed the similarity values for all pairs of the image vectors
and constructed an adjacency matrix using all of them,
the matrix would take up a large amount of memory.
Hence, we replaced relatively small similarity values
for some pairs of image vectors with zero.
Specifically, we chose $p$ image vectors with the highest similarity to $\a_i$
and built from them the set $N_p(\a_i)$.
We then constructed an \by{n}{n} symmetric matrix $\W$ such that
the $(i,j)$th element $w_{ij}$ is
\begin{equation*}
 w_{ij} =
  \left\{
  \begin{array}{ll}
   s(\a_i, \a_j) & \mbox{if} \ \a_i \in N_p(\a_j) \ \mbox{or} \ \a_j \in N_p(\a_i), \\
   0             & \mbox{otherwise}.
  \end{array}
  \right.
\end{equation*}
In the subsequent discussion, 
we will call $p$ the \emph{neighbor size} and 
$N_p(\a_i)$ the \emph{$p$-nearest neighbor set} of $\a_i$.

The experiments used two measures, AC and NMI, 
to evaluate how closely the clusters found
by the algorithms matched the classes of each dataset.
Here, recall that AC stands for accuracy and NMI for normalized mutual information.
We are given $n$ images indexed by integers $1, \ldots, n$ in a dataset
that have been manually classified into $k$ classes $C_1, \ldots, C_k \subset \{1, \ldots, n\}$.
For clusters $T_1, \ldots, T_k$ returned by an algorithm,
we take a permutation $\sigma: \{1,\ldots,k\} \rightarrow \{1,\ldots,k\}$ 
that maximizes $\sum_{i=1}^{k} |C_i \cap T_{\sigma(i)} |$.
AC is defined by
\begin{equation*}
 \frac{1}{n} \sum_{i=1}^{k} |C_i \cap T_{\sigma(i)} |
\end{equation*}
for the $\sigma$.
Note that the problem of finding such a permutation $\sigma$ is an assignment problem,
and it is easily solvable.
Let $\Gamma_1 = \{C_1, \ldots, C_k\}$ and $\Gamma_2 = \{T_1, \ldots, T_k\}$.
NMI is defined by
\begin{equation*}
 \frac{2 \cdot I(\Gamma_1; \Gamma_2)}{H(\Gamma_1) + H(\Gamma_2)}.
\end{equation*}
Here, $I(\Gamma_1; \Gamma_2)$ is the mutual information of $\Gamma_1$ and $\Gamma_2$,
and $H(\Gamma_1)$ and $H(\Gamma_2)$ are the entropies of $\Gamma_1$ and $\Gamma_2$.
For details,
we refer readers to Section 16.3 of the textbook \cite{Man08}.
The values of AC and NMI range from $0$ to $1$.
A higher value indicates a higher degree of matching
between clusters  and classes. 
In particular, if there is a permutation $\sigma$ such that $C_i = T_{\sigma(i)}$
for every $i=1, \ldots, k$, then AC and NMI are each $1$.
Besides AC and NMI, we measured the MCC and the elapsed time of the algorithm.

For each dataset, we constructed adjacency matrices using $p$-nearest neighbor sets
by changing the neighbor size $p \in \{10, 100, 200, 300, 400, 500\}$.
In KSC and GNMF, we repeated the $k$-means++ algorithm $100$ times for each input.

\begin{table}[p]
 \caption{AC, NMI and MCC of algorithms for each dataset in case of $p=300$: AC (top), NMI (middle) and MCC (bottom).
 The columns labeled ``Average'' list the averages of measurements
 over the $100$ trials of the $k$-means++ algorithm, and 
 those labeled ``Worst'' and ``Best'' list the worst and best measurements.}
 \label{Tab: AC, NMI and MCC in 2nd exp}
 \centering

 \begin{tabular}{l c c cc ccc cc ccc}
  \toprule
  \multicolumn{13}{c}{AC} \\
  \midrule
                    & & ELLI &  & & \multicolumn{3}{c}{KSC} &  & & \multicolumn{3}{c}{GNMF}  \\
  \cline{6-8}  \cline{11-13}
                    & &      &  & & Average & Worst & Best & & & Average & Worst & Best \\
  \midrule
  EMNIST-Balanced  & & 0.387 &  & & 0.376 & 0.343 & 0.405  &  &  & 0.087 & 0.084 & 0.090 \\ 
  ETL  & & 0.206 &  & & 0.193 & 0.184 & 0.200  &  &  & 0.030 & 0.030 & 0.031 \\ 
  Fashion-MNIST  & & 0.554 &  & & 0.527 & 0.347 & 0.565  &  &  & 0.544 & 0.524 & 0.584 \\
  MNIST  & & 0.631 &  & & 0.602 & 0.472 & 0.689  &  &  & 0.444 & 0.369 & 0.487 \\
  NDL  & & 0.779 &  & & 0.658 & 0.567 & 0.723  &  &  & 0.217 & 0.203 & 0.230 \\
  \bottomrule
 \end{tabular}

 \bigskip \bigskip
 \begin{tabular}{l c c cc ccc cc ccc}
  \toprule
  \multicolumn{13}{c}{NMI} \\
  \midrule
                    & & ELLI &  & & \multicolumn{3}{c}{KSC} &  & & \multicolumn{3}{c}{GNMF}  \\
  \cline{6-8}  \cline{11-13}
                    & &      &  & & Average & Worst & Best & & & Average & Worst & Best \\
  \midrule
  EMNIST-Balanced  & & 0.515 &  & & 0.513 & 0.502 & 0.525  &  &  & 0.167 & 0.158 & 0.175 \\ 
  ETL  & & 0.329 &  & & 0.328 & 0.319 & 0.332  &  &  & 0.010 & 0.010 & 0.010 \\
  Fashion-MNIST  & & 0.631 &  & & 0.617 & 0.509 & 0.643  &  &  & 0.623 & 0.611 & 0.638 \\ 
  MNIST  & & 0.664 &  & & 0.659 & 0.605 & 0.700  &  &  & 0.531 & 0.502 & 0.569 \\ 
  NDL  & & 0.883 &  & & 0.829 & 0.802 & 0.851  &  &  & 0.510 & 0.500 & 0.519 \\
  \bottomrule
 \end{tabular}

\bigskip \bigskip 
 \begin{tabular}{l c c cc ccc cc ccc}
  \toprule
  \multicolumn{13}{c}{MCC} \\
  \midrule
                    & & ELLI &  & & \multicolumn{3}{c}{KSC} &  & & \multicolumn{3}{c}{GNMF}  \\
  \cline{6-8}  \cline{11-13}
                    & &      &  & & Average & Worst & Best & & & Average & Worst & Best \\
  \midrule
  EMNIST-Balanced  & & 0.510 &  & & 0.567 & 0.703 & 0.495  &  &  & 0.936 & 0.942 & 0.928 \\ 
  ETL  & & 0.789 &  & & 0.725 & 0.858 & 0.694  &  &  & 0.984 & 0.993 & 0.976 \\ 
  Fashion-MNIST  & & 0.144 &  & & 0.214 & 0.543 & 0.146  &  &  & 0.276 & 0.428 & 0.191 \\ 
  MNIST  & & 0.239 &  & & 0.257 & 0.435 & 0.197  &  &  & 0.793 & 0.839 & 0.281 \\ 
  NDL  & & 0.427 &  & & 0.733 & 0.853 & 0.601  &  &  & 0.971 & 0.993 & 0.932 \\
  \bottomrule
 \end{tabular}
\end{table}

\begin{table}[h]
 \caption{Elapsed time of algorithms in seconds  for each dataset in case of $p=300$.
 The columns labeled ``KSC'' and ``GNMF'' list the averages 
 over the $100$ trials of the $k$-means++ algorithm.} 
 \label{Tab: elapsed time in 2nd exp}
 \centering

 \begin{tabular}{l rrrrrr}
  \toprule
                  &  & ELLI  &  & KSC   &  & GNMF  \\
  \midrule
  EMNIST-Balanced & & 100.0 & & 96.8 & & 567.7  \\ 
  ETL & & 53.4 & & 45.5 & & 364.4  \\
  Fashion-MNIST & & 19.3 & & 17.3 & & 59.8  \\
  MNIST & & 15.0 & & 14.1 & & 50.6  \\ 
  NDL & & 32.2 & & 31.2 & & 274.5  \\
  \bottomrule
 \end{tabular}
\end{table}

Tables \ref{Tab: AC, NMI and MCC in 2nd exp} and \ref{Tab: elapsed time in 2nd exp} show the experimental results for $p = 300$:
AC, NMI and MCC  from top to bottom in Table \ref{Tab: AC, NMI and MCC in 2nd exp} 
and elapsed time in Table \ref{Tab: elapsed time in 2nd exp}.
In KSC and GNMF, we repeated the $k$-means++ algorithm $100$ times for each input.
Table \ref{Tab: AC, NMI and MCC in 2nd exp} lists the average, worst, and best values of ACs, NMIs, and MCCs 
in the columns labeled ``Average'', ``Worst'', and ``Best''.
Table \ref{Tab: elapsed time in 2nd exp} lists the averages of the elapsed time in the columns labeled ``KSC' and ``GNMF''.
By comparing the measurements of ELLI with the averages of the measurements of KSC and GNMF,
we can make the following observations.
\begin{itemize}
 \item ELLI and KSC outperform GNMF in terms of AC and NMI,
       except in the case of KSC on Fashion-MNIST.
       The AC and NMI of ELLI are higher than the average AC and NMI of KSC.
       However, the differences are small, except for those on NDL.
 \item ELLI and KSC outperform GNMF in terms of MCC.
       The MCC of ELLI is lower than the average MCC of KSC, except for that on ETL.
 \item ELLI and KSC are faster than GNMF.
       The elapsed time of ELLI is slightly longer than the average elapsed time of KSC.
\end{itemize}
Hence, the experimental results 
imply that the AC and NMI of ELLI can reach at least the average AC and NMI of KSC.
This is apparently an advantage of ELLI over KSC.
After performing the $k$-means method multiple times in KSC,
it is necessary to appropriately select one of the outputs,
which may not be an easy task.
In fact, Table \ref{Tab: AC, NMI and MCC in 2nd exp} shows
that there is a gap between the worst and best values of ACs and NMIs of KSC.
The experimental results on synthetic and real data
also imply that ELLI will often outperform KSC in terms of MCC.

The experimental results for the cases other than $p=300$ show a similar tendency.
Figure \ref{Fig: results of 2nd exp} plots
the ACs, NMIs, and MCCs of the algorithms run on each dataset for neighbor sizes $p \in \{10, 100, 200, 300, 400, 500\}$.
We see that, even if $p$ changes in the range,  
ELLI outperforms KSC in terms of AC and NMI on NDL 
and is about equal to KSC in terms of AC and NMI on the other dataset.
Moreover, the MCC of ELLI is lower than the average MCC of KSC, except for that on ETL. 

\begin{remark}
The previous version of this paper posted on arXiv 
chosen an incorrect value of $k = 46$ for EMNIST-Balanced
during the experiments.
The current paper reports the experimental results for the dataset obtained with the correct value of $k = 47$. 
\end{remark}

\begin{figure}[p]
 \includegraphics[width=\linewidth]{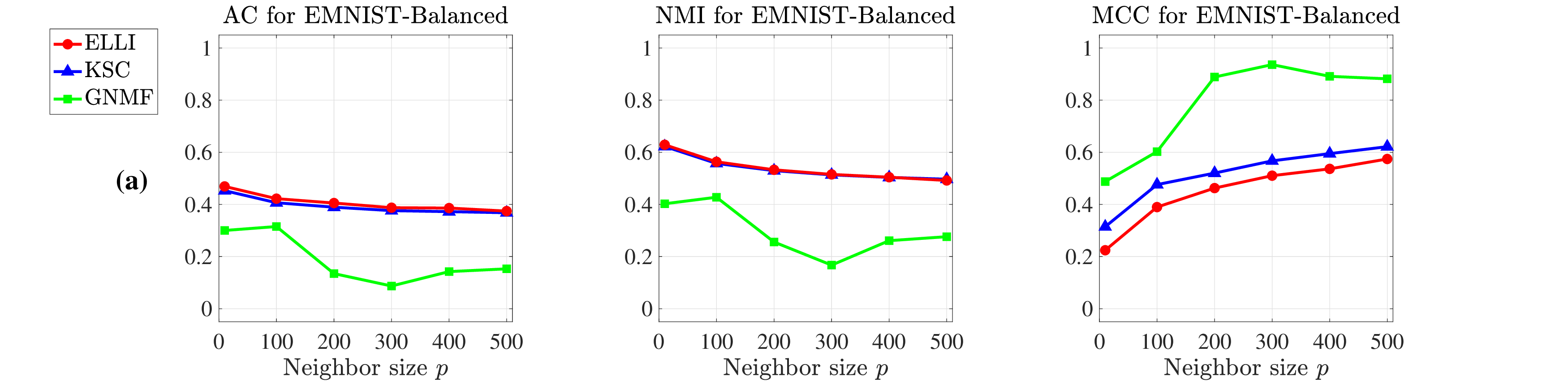}

 \smallskip
 \includegraphics[width=\linewidth]{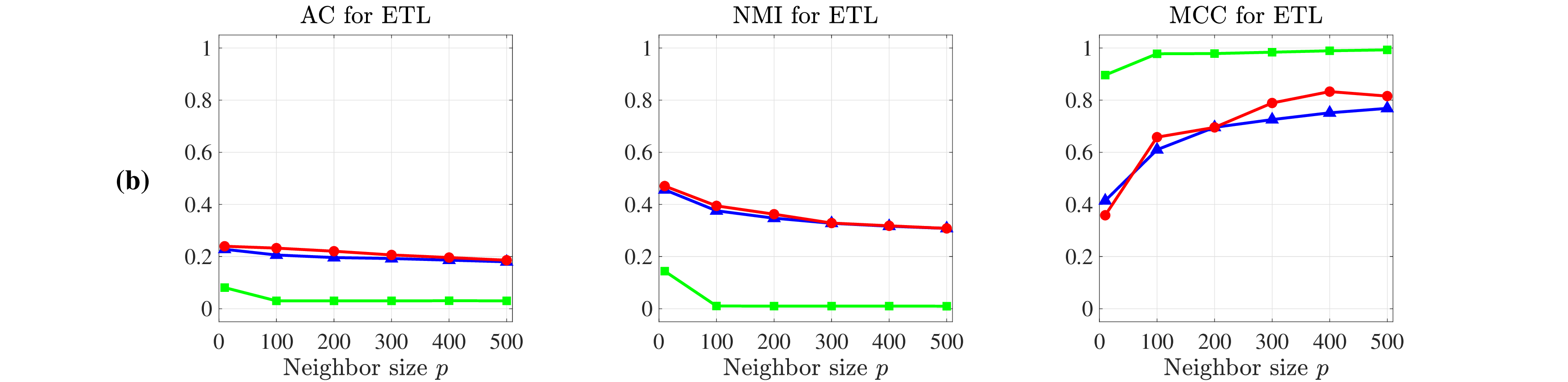}

 \smallskip
 \includegraphics[width=\linewidth]{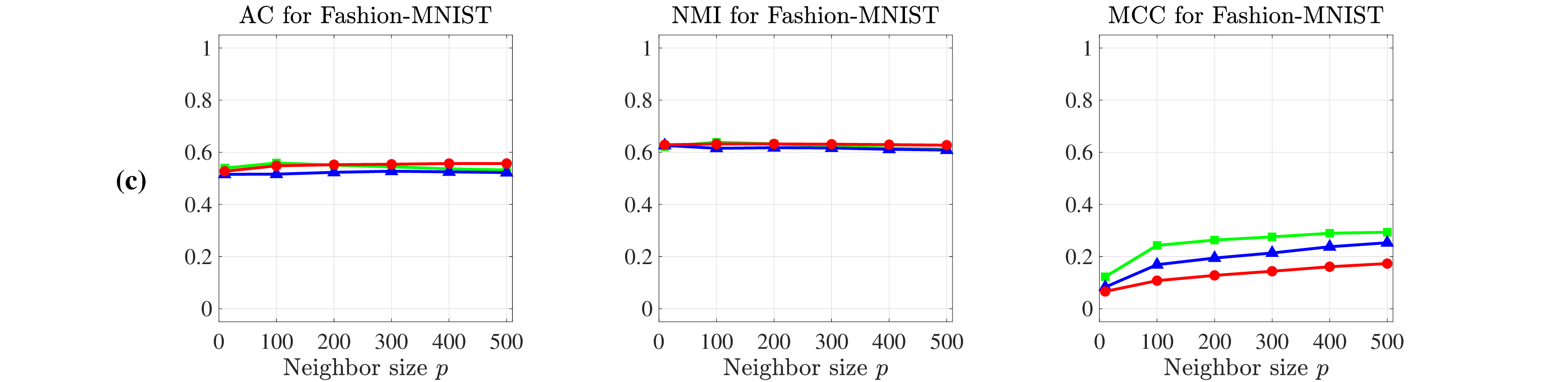}

 \smallskip
 \includegraphics[width=\linewidth]{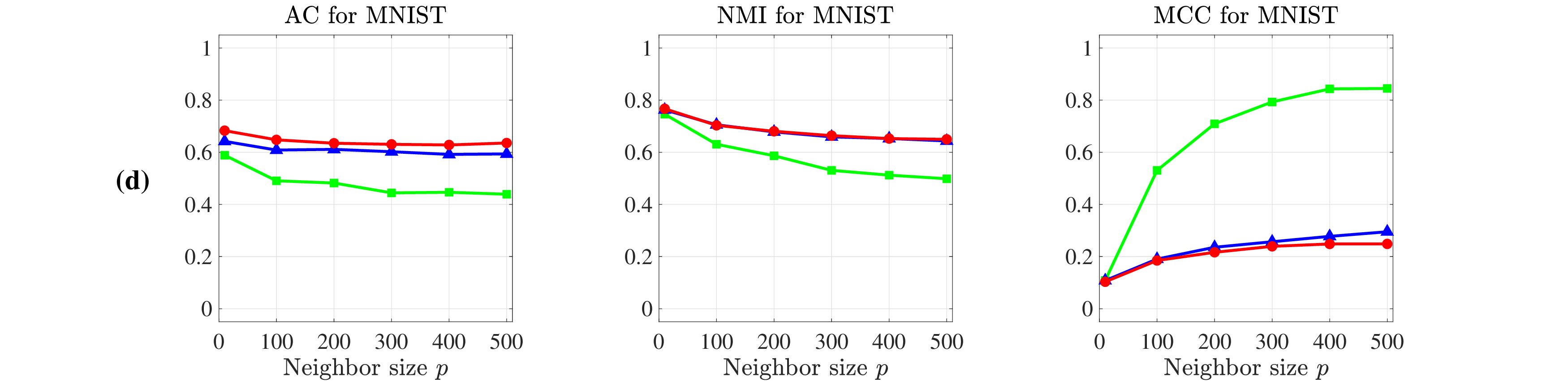}

 \smallskip
 \includegraphics[width=\linewidth]{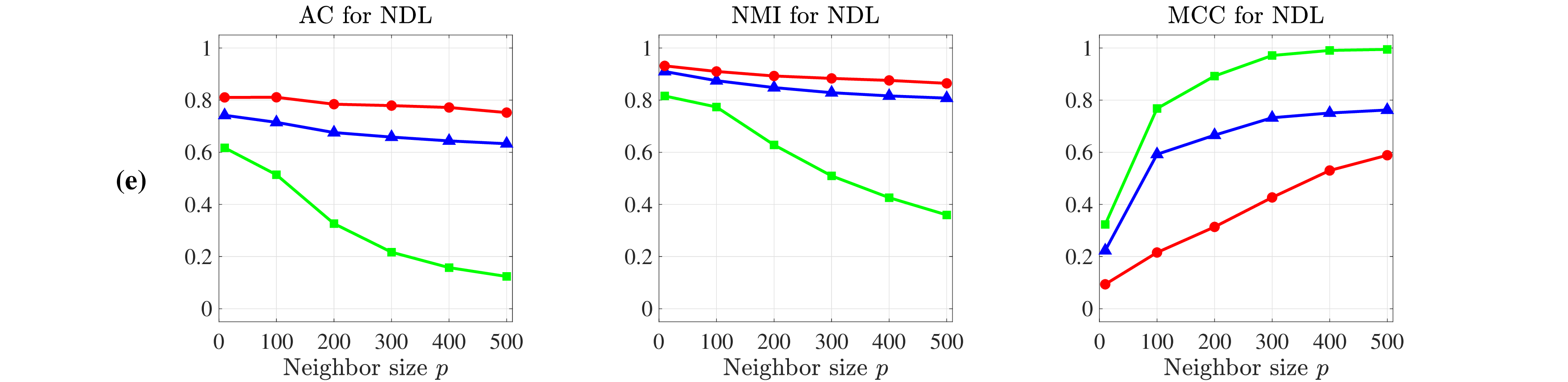}
 \caption{AC, NMI and  MCC of algorithms
 run on each dataset with varying neighbor sizes $p \in \{10, 100, 200, 300, 400, 500\}$:
 (a) EMNIST-Balanced, (b) ETL, (c) Fashion-MNIST, (d) MNIST, and (e) NDL.
 From left to right, figures display AC, NMI and MCC.
 The red points indicate the measurements of ELLI,
 the blue triangles indicate the average of the measurements of KSC
 over the $100$ trials of $k$-means++,
 and the green squares indicate the average of the measurements of GNMF
 over the $100$ trials of $k$-means++.}
 \label{Fig: results of 2nd exp}
\end{figure}

\section{Discussion and Future Research}

There remain issues that need to be addressed.
In Theorem \ref{Theo: main result}, we showed
the range of $\Upsilon$ to ensure that the output of ELLI coincides with 
an optimal $k$-way partition of a graph.
It is unclear whether the range can be further improved.
In the experiments on the image datasets, we experienced that
ELLI did not always achieve significantly  higher AC and NMI than those of KSC
even when the conductance of the clusters returned by ELLI was lower than that of KSC.
The main cause of this unfavorable situation could be that clusters with low conductance in a graph 
do not sufficiently capture the characteristics of manually assigned dataset classes.
This situation may be ameliorated by
revising the way of constructing adjacency matrices from image vectors.
We close this paper by suggesting directions of study for future research.
\begin{itemize}
 \item As explained in Section \ref{Sec: connection with separable NMFs},
       it should be possible to replace the use of an ellipsoid in Step 2 of ELLI
       with an algorithm for solving separable NMF problems.
       In particular, SPA and SNPA are fit for the purpose.
       It would be interesting to explore whether 
       a similar result as Theorem \ref{Theo: main result} can be obtained
       for using either of them instead of an ellipsoid.

 \item 
       Spectral clustering needs to construct adjacency matrices from datasets.
       The construction for large-scale datasets is time and memory consuming.
       To address the issue, the authors of \cite{Che11, Hua20} proposed the use of a bipartite graph 
       for representing the similarities between data.
       It would be interesting to investigate the performance of ELLI incorporated with this technique.

 \item We believe that ELLI works on hyperspectral unmixing problems.
       This problem asks one to find the spectra of constituent materials,
       called endmembers, from a hyperspectral image.
       It can be cast as a graph-based clustering problem.
       Our preliminary experiments often indicated that 
       the index set $I$ found by Step 2 of ELLI provides a good estimate of endmembers.

\end{itemize}

\section*{Acknowledgements}
The author would like to thank the anonymous reviewers for their helpful comments and suggestions,
and in particular one reviewer for providing Corollary \ref{Coro: obtained from main result}.
This research was supported by the Japan Society for the Promotion of Science
(JSPS KAKENHI Grant Numbers 15K20986, 26242027).

\appendix
\section*{Appendix \quad Proof of Lemma \ref{Lemm: bound on a'b}}

We use the following notation.
Let $\a = [a_1, \ldots, a_k]^\trans \in \Real^k$.
For $i \in \{1, \ldots, k\}$,
we denote by $\a_{\setminus i}$ a subvector obtained by removing the $i$th element $a_i$ from $\a$,
i.e.,  
\begin{align*}
\a_{\setminus i} = [a_1, \ldots, a_{i-1}, a_{i+1}, \ldots, a_k]^\trans \in \Real^{k-1}. 
\end{align*}
In a similar way, for $i,j \in \{1, \ldots, k\}$ with $i \neq j$, 
we denote by $\a_{\setminus i, j}$  a subvector obtained by removing the $i$th element $a_i$ and 
the $j$th element $a_j$ from $\a$. 
The proof of the lemma often makes use of the inequality, 
\begin{equation} \label{Exp: norm of a setminus u}
 \| \a_{\setminus i} \|_2 \le \sqrt{1 - \xi^2},
\end{equation}
for $\a \in C(i, \xi)$.
It is easy to verify.
Since $\|\a_{\setminus i}\|_2^2 = \|\a\|_2^2 - a_i^2 = 1 - a_i^2$,
we have $a_i^2 = 1 - \| \a_{\setminus i} \|_2^2$.
Also, since $0 \le \xi \le a_i$, we have $\xi^2 \le a_i^2$.
This leads to $\|\a_{\setminus i} \|_2^2 \le 1 - \xi^2$,
which means inequality (\ref{Exp: norm of a setminus u}).

\begin{proof}[Proof of Lemma \ref{Lemm: bound on a'b}]
We prove (a). Let us write $\a^\trans \b$ as 
 \begin{equation*}
  \a^\trans \b = a_i b_i + \a_{\setminus i}^\trans \b_{\setminus i}.
 \end{equation*}
Since $a_i \ge \xi \ge 0$ and $b_i \ge \xi \ge 0$, we have $a_i b_i \ge \xi^2$.
It follows from the Cauchy-Schwarz inequality and inequality (\ref{Exp: norm of a setminus u}) that
 \begin{equation*}
  |\a_{\setminus i}^\trans \b_{\setminus i}| \le \|\a_{\setminus i} \|_2 \|\b_{\setminus i} \|_2 \le 1 - \xi^2.  
 \end{equation*}
Consequently, we have $\a^\trans \b \ge  2 \xi^2 - 1$.
Next, we prove (b).
 Let us write $\a^\trans \b$ as 
\begin{equation*}
 \a^\trans \b = a_i b_i + a_j b_j + \a_{\setminus i, j}^\trans \b_{\setminus i, j}.
\end{equation*}
Since $\|\a\|_2 = 1 $, we have $a_i \le 1$.
Also, from inequality (\ref{Exp: norm of a setminus u}),
we have $b_i^2 \le \|\b_{\setminus j} \|_2^2 \le 1 - \xi^2$.
This leads to $b_i \le \sqrt{1 - \xi^2}$.
Hence, $a_i b_i \le \sqrt{1-\xi^2}$ holds.
Of course, $a_j b_j \le \sqrt{1-\xi^2}$ holds in the same way.
It follows from  the Cauchy-Schwarz inequality and inequality (\ref{Exp: norm of a setminus u}) that
\begin{equation*}
|\a_{\setminus i,j}^\trans \b_{\setminus i,j}| \le \|\a_{\setminus i,j} \|_2 \|\b_{\setminus i,j} \|_2
\le  \|\a_{\setminus i} \|_2 \|\b_{\setminus j} \|_2 \le 1 - \xi^2. 
\end{equation*} 
Consequently, we have $\a^\trans \b \le -\xi^2 + 2\sqrt{1-\xi^2}+1$.
\end{proof}

\bibliographystyle{abbrv}
\bibliography{reference}

\end{document}